\documentclass{article}
 
\usepackage{datetime}

\usepackage{soul}
\usepackage[british]{babel}
\usepackage{amsthm}
\usepackage{amsmath}
\usepackage{amssymb}
\usepackage{mathrsfs}

\newcommand{\spqed}{}

\usepackage{natbib}
\setcitestyle{open={(},close={)}}

\usepackage[margin=3.5cm]{geometry}

\usepackage{flafter}
\usepackage{float}
\usepackage{placeins}
\usepackage{makecell}
\usepackage{multirow}
\usepackage{longtable}

\usepackage{capt-of}

\usepackage{xcolor}
\definecolor{britishracinggreen}{rgb}{0.0, 0.26, 0.15}
\definecolor{brinkpink}{rgb}{0.98, 0.38, 0.5}

\usepackage{textcomp}
\usepackage{url}

\DeclareFontFamily{U}{jkpmia}{}
\DeclareFontShape{U}{jkpmia}{m}{it}{<->s*jkpmia}{}
\DeclareFontShape{U}{jkpmia}{bx}{it}{<->s*jkpbmia}{}
\DeclareMathAlphabet{\mathfrak}{U}{jkpmia}{m}{it}
\SetMathAlphabet{\mathfrak}{bold}{U}{jkpmia}{bx}{it}

\usepackage{enumitem}

\newcommand{\boldit}[1]{%
  { \bfseries \itshape \boldmath{%
        #1
  }}
}%
    
\newcommand{\boldbox}[1]{%
\begin{center}%
    \begin{minipage}{.95\textwidth}%
      \boldit{#1}%
    \end{minipage}%
\end{center}}

\newlength{\myParindent}
\newcommand{\boldfbox}[1]{%
\begin{center}%
  \fbox{%
    \begin{minipage}{.95\textwidth}%
      \boldit{#1}%
    \end{minipage}}%
\end{center}}

\newtheorem{theorem}{Theorem}
\newtheorem{lemma}[theorem]{Lemma}

\newtheorem{remark}[theorem]{Remark}

\theoremstyle{definition}
\newtheorem{definition}[theorem]{Definition}
\newtheorem{example}{Example}

\newcommand{\N}{\mathbb{N}}
\newcommand{\R}{\mathbb{R}}
\newcommand{\RO}{\mathbb{A}}
\newcommand{\RD}{\mathbb{D}}
\newcommand{\RODS}{\mathbb{O}}

\newcommand{\for}{\textrm{for}}
\newcommand{\itfor}{\textit{for}}

\newcommand{\vareps}{\varepsilon}

\newcommand{\pre}{\preceq}

\newcommand{\longto}{\longrightarrow}
\newcommand{\lin}{{\,\!\in\!\,}}

\newcommand{\mfrak}[1]{\mathfrak{#1}}
\newcommand{\mcal}[1]{\mathcal{#1}}
\newcommand{\mscr}[1]{\mathscr{#1}}
\newcommand{\norm}[1]{{\left\| {#1} \right\|}}
\newcommand{\abs}[1]{{\left| {#1} \right|}}
\newcommand{\card}[1]{{\abs{#1}}}
\newcommand{\lorder}[1][]{%
  \ifthenelse{\equal{#1}{}}%
             {{\le_{\alpha}}}%
             {{\le_{{#1}}}}%
}
\newcommand{\quotient}[2]{{{#1}/\!{#2}}}

\newcommand{\image}[1]{{im\!\left({#1}\right)}}

\newcommand{\D}{\mathcal{D}}
\newcommand{\PD}{\mathcal{PD}}
\newcommand{\U}{\mathcal{U}}
\newcommand{\M}{\mathcal{M}}

\DeclareMathOperator{\HC}{\mathcal{HC}}
\newcommand{\Link}{{\mcal{L}}}

\newcommand{\SLink}{{\mcal{SL}}}
\newcommand{\ALink}{{\mcal{AL}}}
\newcommand{\CLink}{{\mcal{CL}}}

\newcommand{\opt}{{opt}}

\newcommand{\noncmp}{{\perp}}

\DeclareMathOperator*{\argmin}{arg\,min}

\DeclareMathOperator{\diam}{diam}
\DeclareMathOperator{\sep}{sep}

\newcommand{\subs}{\subseteq}
\newcommand{\subp}{\Subset}

\newcommand{\pow}[1]{{\mcal{P}\!\left({#1}\right)}}
\newcommand{\Eqs}[1]{{\mfrak{R}\!\left(#1\right)}}
\newcommand{\Part}[1]{{\mfrak{P}\!\left(#1\right)}}
\newcommand{\ceil}[1]{{\left\lceil{#1}\right\rceil}}

\newcommand{\ultrax}[1]{{\mathfrak{#1}}}
\newcommand{\ultra}{{\ultrax{u}}}
\newcommand{\Ultra}{{\ultrax{U}}}

\newcommand{\la}{\langle}
\newcommand{\ra}{\rangle}

\newcommand{\ari}{{ARI}}
\newcommand{\mari}{{$\ari$}}
\newcommand{\oari}{{\bar{o}ARI}}
\newcommand{\moari}{{$\oari$}}
\newcommand{\loops}{{\mathrm{loops}}}
\newcommand{\mloops}{{$\loops$}}

\newcommand{\ER}{{Erd\H{o}s-R\'{e}nyi}}
\DeclareMathOperator{\E}{{\mathbb{E}}}

\newcommand{\deft}[1]{\textbf{\boldmath{#1}}}
\newcommand{\CarlssonMemoli}{{Carlsson and M\'{e}moli}}

\title{Order preserving\\hierarchical agglomerative clustering}
\author{Daniel Bakkelund\\{\tt daniel.bakkelund@ifi.uio.no}}
\date{}

\usepackage{wrapfig} 
\usepackage{tikz-cd}

\usepackage{pgfplots}
\usepackage{pgfplotstable}
\pgfplotsset{width=7cm,compat=1.17}
\usepgfplotslibrary{statistics}
\usepgfplotslibrary{fillbetween}

\definecolor{mylinkcolor}{rgb}{0.0, 0.0, 0.5}
\usepackage[
  colorlinks=true,
  allcolors=mylinkcolor,
  allbordercolors=white]{hyperref}

\begin{document}
\maketitle

\begin{abstract}
  Partial orders and directed acyclic graphs are commonly recurring data structures that
arise naturally in numerous domains and applications and are used to represent ordered relations between
entities in the domains. Examples are task dependencies in a project plan, transaction order in distributed
ledgers and execution sequences of tasks in computer programs, just to mention a few.
We study the problem of \emph{order preserving hierarchical clustering} of this kind of ordered data.
That is, if we have $a<b$ in the original data and denote their respective clusters 
by $[a]$ and $[b]$, then we shall have $[a]<[b]$ in the produced clustering.
The clustering is similarity based and uses standard linkage functions, such as single- and
complete linkage, and is an extension of classical hierarchical clustering.

To achieve this, we define the output from running classical hierarchical clustering on strictly 
ordered data to be \emph{partial dendrograms}; sub-trees of classical dendrograms with 
several connected components. We then construct an embedding of partial dendrograms over 
a set into the family of ultrametrics over the same set. 
An optimal hierarchical clustering is defined as
the partial dendrogram corresponding to the ultrametric closest to the original
dissimilarity measure, measured in the $p$-norm. 
Thus, the method is a combination of classical hierarchical clustering and ultrametric fitting.

A reference implementation is employed for experiments on both synthetic random data and real world
data from a database of machine parts. When compared to existing methods, the experiments show that
our method excels both in cluster quality and order preservation.

\end{abstract}

\newcommand{\kwdsep}{$\,\cdot\,$}

\begin{center}
\begin{minipage}{.88\textwidth}
{\flushleft {\bf Keywords} Hierarchical clustering\kwdsep
Order preserving clustering\kwdsep
Partial dendrogram\kwdsep
Unsupervised classification\kwdsep
Ultrametric fitting\kwdsep
Acyclic partition
}
\end{minipage}
\end{center}

\section{Introduction}

Clustering is one of the oldest and most frequently used techniques for exploratory data analysis
and unsupervised classification.
The toolbox contains a large variety of methods and algorithms, spanning from the initial, but
still popular ideas of $k$-means \citep{Macqueen1967} and hierarchical 
clustering \citep{Johnson1967}, 
to more recent methods, such as density- and model based clustering \citep{KriegelEtAl2011,FraleyRaftery2002},
and semi-supervised methods \citep{BasuDavidsonWagstaff2008}, plus a large list of variants.
All these methods have one thing in common: they try to extract hidden structure from the data,
and make it visible to the analyst. But they also share another feature: if the analysed data
is already endowed with some form of structure, the structure is lost in the clustering process;
the clustering does not try to retain the structure.

In this paper, we show how to extend hierarchical clustering to relational data in a way that
preserves the relations. In particular, if the input is a set $X$ equipped with a strict partial order $<$,
and if $a,b \in X$, we ensure that if $a<b$ then we will have $[a] <' [b]$ after clustering,
where $[a]$ and $[b]$ are the respective clusters of $a$ and $b$, and $<'$ is a partial order
on the clusters naturally induced by $<$.

Since directed acyclic graphs (DAGs) correspond to partial orders, our method works equally well for
DAGs. If the input is a DAG, then every clustering in the produced hierarchy is a DAG of clusters,
and there exists a DAG homomorphism from the original DAG to the cluster DAG.

\subsection{Motivating real-world use case} 
\label{section:motivating-use-case}
The motivation for our method comes from an industry database of machine parts
that are arranged in part-of relations: parts are registered as sub-parts of other parts.
For historical reasons, there have been incidents of copy-paste of machine designs,
and the copies have been given entirely new identifiers with no links to the original design.
In hindsight, there is a wish to identify these equivalent machine parts, but telling them apart is hard. 
Also, the metadata that is available has a tendency of displaying high similarity 
between a part and its sub-parts, leading to ``vertical clustering'' in the data. 

Since the motivation is to identify equivalent machinery with the aim of replacing one piece of
machinery with an equivalent part, and since a part and its sub-parts by no means
can be interchanged, it is essential to maintain this parent-child relationship.
Moreover, since a part and its sub-part are never equivalent, this is a strict order
relation. The set of all machine parts thus makes up a strictly partially ordered set.
By preserving these relations in the clustering process,
we can eliminate the errors due to close resemblance between the part and the sub-part,
resulting in improved over all quality of the clustering.

\smallskip

This is but one concrete example of a real world problem where the method we present performs
significantly better than standard methods that disregard the structure. It is possible to imagine
several other cases for which we have not yet had the opportunity to test our methodology.
We will only mention two here; citation network analysis and time series alignment:

\smallskip

Citation networks are partial orders, where the order is defined by the citations. If we perform
order preserving clustering in the above sense on citation networks, the clusters will contain related
research, and the clusters will be ordered according to appearance relative other related research.
This differs from clustering with regards to time: when clustering with time as a parameter,
you have to choose, implicitly or explicitly,
a time interval for each cluster. When the citation graph is used for ordering, the clusters will contain
research that occurred in parallel, citing similar sources, and being cited by similar sources, regardless
to whether they occurred in some particular time interval.

\smallskip

A time series is a totally ordered set of events, so that a family of time series is a partially ordered
set. Assume that you want to do time series alignment, matching events from one time series with events
from another, but for some reason the time stamps are corrupted and cannot be used for this purpose.
Given a measure of (dis-)similarity between events, we can cluster the events to figure out which events
are the more similar. Since an optimal order preserving clustering is one
that both preserves all event orders and matches the most similar events across the time series,
ideally the result is a series of clusters with each cluster containing
the events that correspond to each other across the time series.

\subsection{Problem overview}
\label{section:problem-overview}
\label{section:HC-at-a-glance}
\label{section:introducing-a-strict-order}
 
Given a set $X$ together with a notion of \hbox{(dis-)similarity} between the elements of~$X$, 
a hierarchical agglomerative clustering can be obtained as follows \citep[\S 3.2]{JainDubes1988}:
\begin{enumerate}
\item Start by placing each element of $X$ in a separate cluster.
\item Pick the two clusters that are most similar according to the (dis-)similarity measure, 
  and combine them into one cluster by taking their union.
  \label{step:glance-pick}
\item If all elements of $X$ are in the same cluster, we are done. 
  Otherwise, go to Step~\ref{step:glance-pick} and continue.
\end{enumerate}
The result from this process is a dendrogram; a tree structure showing the sequence of the clustering
process (Figure~\ref{fig:dendrogram}).

\vbox{\begin{center}
  \input{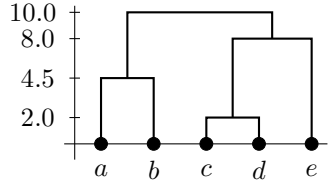}
  \vspace{-10px}
  \captionof{figure}{A dendrogram over the set $X=\{a,b,c,d,e\}$.
    The elements of $X$ are the leaf nodes of the dendrogram, and, starting at the bottom, the horizontal
    bars indicate which elements are joined at which step in the process. The numbers on the $y$-axis
    indicate at which dissimilarity level the different clusters were formed.}
  \label{fig:dendrogram}
\end{center}}



Now, given a partially ordered set $X=\{a,b,c,d\}$ where $a<b$ and $c<d$,
we can use arrows to denote the order relation, thinking of $X$ as
a directed acyclic graph with two connected components.
If we want to produce a hierarchical clustering of $X$, while at the same time
maintaining the order relation, our options are depicted in the Hasse digram in 
Figure~\ref{fig:meet-semilattice}.

\begin{figure}[htbp]
\vbox{\begin{center}
  \input{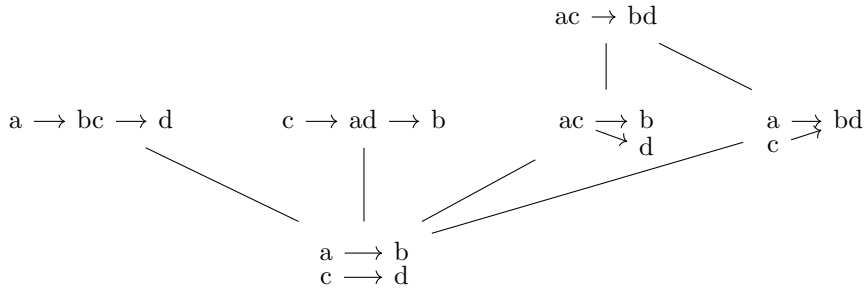}
  \vspace{-10px}
  \captionof{figure}{Possible order preserving hierarchical clusterings over the
    set $X=\{a,b,c,d\}$ with $a<b$ and $c<d$. Adjacent elements indicate clusters.}
  \label{fig:meet-semilattice}
\end{center}}
\end{figure}

Each path in this diagram, starting at the bottom and advancing upwards, represents a hierarchical
clustering. But, since we are required to preserve the strict order relation, we cannot merge any more
elements than what we see here. This means that we will never obtain dendrograms like the one in 
Figure~\ref{fig:dendrogram}, that joins at the top when all elements are placed in a single cluster.
Rather, the output of hierarchical agglomerative clustering would take the form of
\emph{partial dendrograms} like those of Figure~\ref{fig:partial-dendrograms-intro}.

\begin{figure}[htpb]
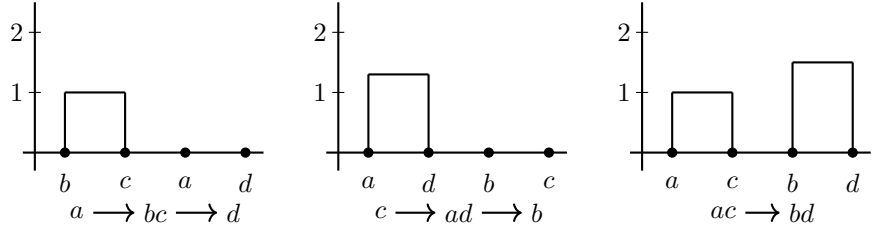

\vbox{\begin{center}
    \begin{tabular}{ccc}
      \input{partial_dendrogram_a_bc_d} &
      \input{partial_dendrogram_c_ad_b} &
      \input{partial_dendrogram_ac_bd}   
    \end{tabular}
    \captionof{figure}{Partial dendrograms over the set $X=\{a,b,c,d\}$ with $a<b$ and $c<d$.
      Each partial dendrogram corresponds to a path in Figure~\ref{fig:meet-semilattice} 
      starting at the bottom and advancing upwards to 
      the ordered set depicted below the dendrogram.}
    \label{fig:partial-dendrograms-intro}
\end{center}}
\end{figure}
 
To complicate matters, if both $(a,d)$ and $(a,c)$ are pairs of minimal dissimilarity,
then they are both candidates for the first merge.
From Figure~\ref{fig:meet-semilattice} we can see that $ad$ and $ac$ 
are mutual exclusive merges, and that choosing one over the other
leads to very different solutions. We therefore need a method to decide which
candidate merge, or which candidate partial dendrogram, is the better.

\subsection{Outline of our method and contributions}
\label{section:outline-of-contributions}

As our first contribution, to solve the problem of picking one candidate merge among a set of
tied connections, we present a permutation invariant method for hierarchical agglomerative clustering.
The method uses the classical linkage functions of single-, average- and complete linkage, but
is optimisation based, as opposed to the algorithmic definition of classical hierarchical
clustering. Recalling that every hierarchical clustering corresponds to a unique
ultrametric \citep{Jardine1971}, the
optimisation criterion is that of minimising the matrix norm of the difference between the original
dissimilarity and the ultrametric corresponding to the hierarchical clustering, a method known
as \emph{ultrametric fitting}~\citep{DeSoeteCarrolDeSarbo1987}. 

\smallskip

We have seen that order preserving hierarchical agglomerative clustering on strictly partially ordered sets
leads to partial dendrograms. In order to evaluate the ultrametric fitting of a partial dendrogram,
our next contribution is an embedding of partial dendrograms over a set into the family of ultrametrics
over the same set.

\smallskip

Our main contribution, \emph{order preserving hierarchical agglomerative clustering of strictly
  partially ordered sets}, is the combination of the two. We define an optimal order preserving hierarchical
clustering to be the hierarchical clustering with the partial dendrogram that has the best ultrametric
fit relative the original dissimilarity measure.

\smallskip

In want of an efficient algorithm, we present a method of approximation that can be computed in
polynomial time. We demonstrate the approximation on synthetic data generated as random directed
acyclic graphs and random dissimilarity measures, as well as on data from the parts database 
motivating this research. We evaluate the quality of the obtained clustering by computing the
adjusted Rand index relative a planted partition \citep{HubertArabie1985}.
We provide a novel method for comparing two induced order relations using a modified adjusted Rand index,
which we believe is a first of its kind. We also provide simple method for computing the level of order
preservation of a clustering of an ordered set by counting the number of induced loops.

\smallskip

Beyond our main contribution, we believe that the embedding of partial dendrograms into
ultrametrics may be of interest to a larger audience.
The embedding provides a means for treating partial dendrograms as complete dendrograms, offering
access to the entire rack of tools that already exists in this domain.
An obvious example candidate is that of hierarchical clustering with must-link and no-link
constraints. The no-link constraints will necessarily lead to partial dendrograms that can be
easily evaluated in our framework.

\subsubsection{Summary of contributions}

Our main contribution is the theory for order preserving hierarchical agglomerative
clustering for strict posets. Further contributions we wish to highlight are:
\begin{itemize}
\item A theory for embedding partial dendrograms over a set into the set of complete dendrograms
  over the same set.
\item An optimisation based, permutation invariant hierarchical clustering methodology for non-ordered
  sets that is very similar to classical hierarchical clustering.
\item A polynomial time approximation scheme for order preserving hierarchical agglomerative clustering
\item A novel method for comparison of induced order relations over a set based on the adjusted Rand index.
\item A measure of the level of order preservation of a clustering of an ordered set.
\end{itemize}

\subsection{Related work}
\label{section:related-work}
Hierarchical agglomerative clustering is described in a plethora of books and articles, and we shall
not try to give an account of that material. For an introduction to the subject, 
see \citep[\S 3.2]{JainDubes1988}. 

\subsubsection{Clustering of ordered data}
There are quite a few articles presenting clustering of ordered data, placing themselves in one
of two categories.

The first is clustering of sets where the (dis)similarity measure is replaced by information
about whether one pair of elements is more similar than another pair of elements, for example
based on user preferences. This is sometimes referred to as \emph{comparison based clustering}.
See the recent article by \citet{GhoshdastidarPerrotLuxburg2019} for an example and
references.
In this category, we also find the works of \citet{JanowitzBook2010}, providing a wholly
order theoretic description of hierarchical clustering, including the case where the
dissimilarity measure is replaced by a partially ordered set.

The second variant is to partition a family of ordered sets so that similarly ordered sets
are associated with each other. Examples include the paper by \citet{KamishimaFujiki2003}, 
where they develop a variation of $k$-means, called $k$-$o'$means, 
for clustering preference data, each list of preferences being a totally ordered set.
Other examples in this category include clustering of times series,
identifying which times series are alike \citep{Luczak2016}.

Our method differs from all of the above in that we cluster elements inside one ordered
set through the use of a (dis)similarity measure, while maintaining the original orders of elements.

\subsubsection{Clustering to detect order}

Another variant is the detection of order relations in data through clustering:
In \citep{CarlssonEtAl2014}, it is demonstrated how hierarchical agglomerative quasi-clustering
can be used to deduce a partial order of ``net flow'' from an asymmetric network.

In this category, it is also worth mentioning dynamic time warping. This is a method
for aligning time series, and can be considered as clustering across two time series that
is indeed order preserving. See \citep{Luczak2016} for further references on this.

\subsubsection{Acyclic graph partitioning problems}
The problem of order preserving hierarchical agglomerative clustering can be
said to belong to the family of \emph{acyclic graph partitioning problems} \citep{HerrmannEtAl2017}. 
If we consider the strict partial
order to be a directed acyclic graph (DAG), the task is to partition the vertices into groups so that
the groups together with the arrows still makes up a DAG.

Graph partitioning has received a substantial attention from researchers, especially within computer
science, over the last $50$ years. 
Two important fields of application of this theory are VLSI and parallel execution.

In VLSI, short for Very Large Scale Integration, the problem can be formulated as follows:
Given a set of micro processors, the wires that connect them, and a set of circuit boards, how do you
best place the processors on the circuit boards in order to optimise a given objective function?
Typically, a part of the objective function is to minimise the wire length. But other features may also
be part of the optimisation, such as the amount or volume of traffic between certain processors 
etc. \citep{MarkovHuKim2015}

For parallel processing, the input data is a set of tasks to be executed. The tasks are organised 
as a DAG, where predecessors must be executed before descendants. Given a finite number
of processors, the problem is to group the tasks so that they can be run group-wise on a processor, 
or running groups in parallel on different processors, in order to execute all tasks as quickly as possible.
Typically additional information available is memory requirements, expected execution times for the tasks,
etc. \citep{BulucEtAl2016}

It is not difficult to understand why both areas have received attention, 
being essential in the development of modern computers. The development of theory and
methods has been both successful and abundant, and a large array of techniques are available, both academic
and commercially.

Although both problems do indeed perform clustering of strict partial orders, their solutions are not 
directly transferable to exploratory data analysis. Mostly because they have very specific constraints
and objectives originating from their respective problem domains.

The method we propose in this paper has as input a strict partial order (equivalently; a DAG)
together with an arbitrary dissimilarity measure. We then use the classical linkage functions single-,
average-, and complete linkage to suggest clusterings of the vertices from the input dataset, while
preserving the original order relation.

Our method therefore places itself firmly in the family of acyclic graph partitioning methodologies, 
but with different motivation, objective and solution, compared to existing methods.

\subsubsection{Hierarchical clustering as an optimisation problem}
Several publications aim at solving hierarchical clustering in terms of optimisation.
However, due to the procedural nature of classical hierarchical clustering, combined with
the linkage functions, pinning down an objective function may be an impossible task. Especially
since classical hierarchical clustering is not even well defined for complete linkage in the presence
of tied connections. This leads to a general abandonment of linkage functions in optimisation based
hierarchical clustering.

Quite commonly, optimisation based hierarchical clustering is done in terms of \emph{ultrametric fitting}.
That is, it aims to find an ultrametric that is as close to the original dissimilarity measure as possible, 
perhaps adding some additional constraints \citep{GilpinNijssenDavidson2013,ChierchiaPerret2019}.
It is well known that solving single linkage hierarchical clustering is equivalent to finding the 
so called \emph{maximal sub-dominant ultrametric}. That is; the ultrametric that is pointwise maximal
among all ultrametrics not exceeding the original dissimilarity \citep{Rammal1986}.
But for the other linkage functions, there is no equivalent result.

Optimisation based hierarchical clustering therefore generally present alternative definitions
of hierarchical clustering. Quite often based on objective functions that originate from some
particular domain. Exceptions from this are, for example,
Ward's method \citep{Ward1963}, where the topology of the clusters are the focus of the objective,
and also the recent addition by \citet{Dasgupta2016}, where the 
optimisation aims towards topological properties of the generated dendrogram.

Although our method is, eventually, based on ultrametric fitting, we optimise over a very particular
set of dendrograms. Namely the dendrograms that can be generated through 
classical hierarchical clustering with linkage functions. 
It is therefore reasonable to claim that our method places itself 
between classical hierarchical clustering and optimised models.

\subsubsection{Clustering with constraints}
A significant amount of research has been devoted to the topic of clustering with constraints in
the form of pairwise \emph{must-link} or \emph{no-link} constraints, often in addition to
other constraints, such as minimal- and maximal distance constraints, and so on.
Some work as also been
done on hierarchical agglomerative clustering with constraints, starting with the works 
of \citet{DavidsonRavi2005}. For a thorough treatment of constrained clustering,
see \citep{BasuDavidsonWagstaff2008}.

Order preserving clustering (as well as acyclic partitioning) can be seen as a particular version
of constrained clustering, where the constraint is a \emph{directed, transitive cannot-link} constraint.
A type of constraint that is not found in the constrained clustering literature.

\subsubsection{Clustering in information networks}
A large amount of research has been conducted on the problem of clustering nodes in networks,
and a more recent field of research is that of clustering data organised in
\emph{heterogeneous information networks}, or HINs for short \citep{PioEtAl2018}.
A HIN is an undirected graph where both vertices and edges may have different, or
even multiple, types. RDF graphs \citep{Lassila1999} is but one example of HINs.
In a sense, we can say that the availability of multiple
types allow HINs to model the real world more closely, but with the penalty of increased complexity.
It is fair to consider HIN clustering a generalisation of classical network clustering,
where, in the classical setting, all vertices and edges are of one common type.

However, the general case in clustering both classical networks and HINs is that although the network
structure serves to influence the clustering, the structure is usually lost in the clustering. The most
classical example is where connectedness between vertices contribute to vertex similarity, and then the
most connected vertices (clique-like subgraphs) are clustered together. Although this can be seen as a
type of relation preserving clustering, in order preserving clustering, the opposite is taking place:
the more connected two vertices are, the more reason \emph{not} to place them
in the same cluster. Indeed, as we show in Section~\ref{section:order-theory}, for the theory we present
in this paper, two elements can only be clustered together if there are \emph{no} paths connecting them.

An example of HIN clustering that \emph{is} structure preserving is \citep{LiEtAl2017}.
A HIN comes with a schema, or a \emph{schematic graph}, describing which types are related to which
other types. For \citet{LiEtAl2017},
the goal is to cluster each set of same-type nodes according to a discovered similarity measure.
The result is thus a schematic graph where each node is a clustering of
vertices of the same type. This differs from the problem we study in that
we do not know which elements are of the same type; to discover this is the goal of the clustering.
Hence, the problems are similar but different; we could rephrase our problem as that of \emph{deriving} a
directed schematic graph from unlabeled vertices, where each vertex in the schematic graph is a set of
equivalent machine parts, and the directed edges are the part-of relations.

\subsection{Organisation of the remainder of this paper} 
\label{section:paper-overview}

Section~\ref{section:background} provides necessary background material. 

In Section~\ref{section:optimised-hc}, we develop
\emph{optimised hierarchical agglomerative clustering} for non-ordered sets; our
permutation invariant clustering model that is tailored especially to fit into our
framework for agglomerative clustering of ordered sets.

In Section~\ref{section:order-theory}, we tackle the problem of order preservation during clustering:
We define what we mean by order preservation, and classify exactly the clusterings that are order preserving.
We also provide concise necessary and sufficient conditions for an hierarchical agglomerative
clustering algorithm to be order preserving.

Section~\ref{section:partial-dendrograms} defines partial dendrograms and develops the embedding of
partial dendrograms over an ordered set into the family of ultrametrics over the same set. 

Our main result, \emph{order preserving hierarchical agglomerative clustering for strict partial orders},
is presented Section~\ref{section:HC}.

Section~\ref{section:approximation} provides a polynomial time approximation scheme for our
method, and Section~\ref{section:approximation-demo} demonstrates the efficacy of the approximation on
synthetic data.

Section~\ref{section:demo} presents the results from applying our approximation method to a subset of the
data in the parts database, comparing with existing methods,
and finally, Section~\ref{section:conclusions} closes the article with some concluding remarks, and a list of
future work topics.

\section{Background}
\label{section:background}

In this section we recall basic background material.
We start by recollecting the required order-theoretical tools together with equivalence relations,
before recalling classical hierarchical clustering.

\subsection{Relations}

\begin{definition}
A \deft{relation} $R$ on a set $X$ is a subset $R \subs X \times X$, and we say that $x$ and $y$ are 
\deft{related} if $(x,y) \in R$. The short hand notation 
$a R b$ is equivalent to writing $(a,b) \in R$.
\end{definition}

\subsubsection{Strict and non-strict partial orders}

A \deft{strict partial order} on a set $X$ is a relation $S$ on $X$ that is
irreflexive and transitive.
Recall that, an irreflexive and transitive 
relation is also anti-symmetric. 
A \deft{strictly partially ordered set}, or a \deft{strict poset}, 
is a pair $(X,S)$, where $X$ is a set and $S$ is a strict partial
order on $X$. We commonly denote a strict partial order by the symbol $<$. 

On the other hand a \deft{partial order} on $X$ is a relation $P$ on $X$ that is
reflexive, asymmetric and transitive,
and the pair $(X,P)$ is called a \deft{partially ordered set}, or a \deft{poset}. 
The usual notation for a partial order is $\le$.

We shall just refer to strict and non-strict partial orders as \emph{orders}, unless there is
any need for disambiguation:
If $R$ is an order on $X$, we say that $a,b \in X$ are \deft{comparable}
if either $(a,b) \in R$ or $(b,a) \in R$. And, if every pair of elements in $X$ are comparable, 
we call $X$ \deft{totally ordered}.
A totally ordered subset of an ordered set is called a \deft{chain}, and a subset where no two
elements are comparable is called an \deft{antichain}.
We denote \deft{non-comparability} by $a \noncmp b$. That is, for any elements $a,b$ in an 
antichain, we have $a \noncmp b$.

A \deft{cycle} in a relation $E$ is a sequence in $E$ on
the form $(a,b_1),(b_1,b_2),\ldots,(b_n,a)$.
The \deft{transitive closure} of $E$ is the minimal set
$\overline{E}$ for which the following holds: If there is a sequence of pairs
$(a_1,a_2),(a_2,a_3),\ldots,(a_{n-1},a_n)$ in $E$, then $(a_1,a_n) \in \overline E$.

Let $(X,E)$ be an ordered set. An element $x_0 \in X$ is a \deft{minimal element} 
if there is no element $y \in X-\{x_0\}$
for which $(y,x_0) \in E$. Dually, $y_0$ is a \deft{maximal element} if there is no $x \in X-\{y_0\}$
for which $(y_0,x) \in E$.
If $(X,E)$ has a unique minimal element, then this is called the \deft{bottom element} 
or the \deft{least element}, and a unique maximal element is called the \deft{top element} or 
the \deft{greatest element}.

Finally, a map $f : (X,<_X) \to (Y,<_Y)$ is \deft{order preserving} if 
$a <_X b \, \Rightarrow \, f(a) <_Y f(b)$, and if $f$ is a set isomorphism (that is, a bijection) for which $f^{-1}$
is also order preserving, we say that $f$ is an \deft{order isomorphism}, and that the sets
$(X,<_X)$ and $(Y,<_Y)$ are \deft{order isomorphic}, writing $(X,<_X) \approx (Y,<_Y)$.

\subsubsection{Partitions and equivalence relations}
A \deft{partition} of $X$ is a collection of disjoint subsets of $X$, the union of which is~$X$. 
The family of all partitions of~$X$, denoted~$\Part{X}$, 
has a natural partial order defined by partition-refinement:
If $\mcal{A} = \{A_i\}_i$ and $\mcal{B} = \{B_j\}_j$ are partitions of $X$, we say that $\mcal{A}$ is
a \deft{refinement} of $\mcal{B}$, writing $\mcal{A} \subp \mcal{B}$, 
if, for every $A_i \in \mcal{A}$ there exists a $B_j \in \mcal{B}$ such that $A_i \subs B_j$.
The sets of a partition are referred to as \deft{blocks}.

An \deft{equivalence relation} is a relation $\mscr{R}$ on $X$ that is reflexive, symmetric and transitive.
Let the family of all equivalence relations over a set $X$ be denoted by $\Eqs{X}$. 
If $\mscr{R} \in \Eqs{X}$ and $(x,y) \in \mscr{R}$, we say that
$x$ and $y$ are \deft{equivalent}, writing $x \sim y$. The maximal set of elements equivalent
to $x \in X$ is called the \deft{equivalence class of $x$}, and is denoted $[x]$. 
$\Eqs{X}$ is also partially ordered, but by subset inclusion: that is, for
$\mscr{R,S} \in \Eqs{X}$, we say that $\mscr{R}$ is less than or equal to $\mscr{S}$ if and only if 
$\mscr{R} \subs \mscr{S}$.

The \deft{quotient of $X$ modulo $\mscr{R}$}, denoted $\quotient{X}{\mscr{R}}$,
is the set of equivalence classes of~$X$ under~$\mscr{R}$.
Notice that~$[x]$ is an element of $\quotient{X}{\mscr{R}}$, but a subset of~$X$.
Since the equivalence classes are subsets of $X$ that together cover $X$, $\quotient{X}{\mscr{R}}$
is a partition of~$X$ with equivalence classes being the blocks of the partition.
The family of partitions of $X$ is in a one-to-one correspondence with the equivalence relations
of $X$, and the correspondence is order preserving; 
if $\mcal{A} = \quotient{X}{\mscr{A}}$ and $\mcal{B} = \quotient{X}{\mscr{B}}$, we have
\[
\mcal{A} \subp \mcal{B} \ \Leftrightarrow \ \mscr{A} \subs \mscr{B}.
\]

Both $\Part{X}$ and $\Eqs{X}$ have top- and bottom elements:
The least element of $\Part{X}$ is the singleton partition $S(X)$, where each element is in
a block by itself: $S(X)=\{\{x\} \, | \, x \in X\}$. 
The singleton partition corresponds to the diagonal equivalence
relation, given by $\Delta(X) = \{(x,x) \, | \, x \in X\}$, which is the least element of $\Eqs{X}$.
The greatest element of $\Part{X}$ is the trivial partition $\{X\}$, corresponding to the equivalence
relation~$X \times X$, where all element are equivalent. That is
\begin{align*}
S(X) &= \quotient{X}{\Delta(X)} & \text{and} && \{X\} &= \quotient{X}{(X \times X)}.
\end{align*}

If $\mcal{A}$ and $\mcal{B}$ are partitions of $X$ with $\mcal{A}$ being a refinement of $\mcal{B}$, we
say that $\mcal{A}$ is \deft{finer} than $\mcal{B}$, and that $\mcal{B}$ is \deft{coarser} than $\mcal{A}$. 
We use the exact same terminology for the corresponding equivalence relations.

For a subset $A \subs X$, let the notation $\quotient{X}{A}$ denote the partition of $X$
where all of $A$ is one equivalence class, and the rest of $X$ remains as singletons. Formally,
this corresponds to the equivalence relation $\mscr{R}_{\!A} = \Delta(X) \cup (A \times A)$.
And finally, the \deft{quotient map} corresponding to an equivalence relation $\mscr{R} \in \Eqs{X}$
is the unique map $q_\mscr{R}:X \to \quotient{X}{\mscr{R}}$ defined as $q_\mscr{R}(x)=[x]$.
That is, $q_\mscr{R}$ sends each element to its equivalence class.

\subsection{Classical hierarchical clustering} \label{section:classical-HC}
In this section, we recall classical hierarchical clustering in terms of \citet{Jardine1971}.
Our theory builds directly on the theory for classical hierarchical clustering, 
so we need to provide a fair bit of detail, especially in view of the fact that there is a general
lack of standardised notation for hierarchical clustering theory, and that the level of formality in
definitions and notation varies among publications.

We start by recalling the formal definition of a dendrogram,
before recalling dissimilarity measures and ultrametrics. Thereafter, we recall linkage functions,
and at the end of the section, we tie all the concepts together and provide a definition of
classical hierarchical agglomerative clustering.

\begin{definition} \label{def:clustering}
A \deft{clustering} of a set $X$ is a partition of $X$, and a \deft{hierarchical clustering} is a 
chain in $\Part{X}$ containing both the bottom and top elements. 
A \deft{cluster} in a clustering is a block in the partition. 
\end{definition}

\begin{example}
For the three-element space $X=\{a,b,c\}$, the lattice of partitions takes
the form of the below Hasse diagram.
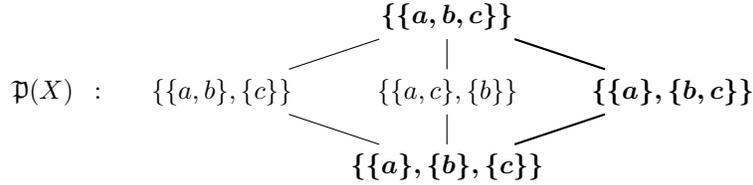
\begin{center}
\input{three_elt_partitions}
\end{center}
The elements in bold make up a chain in $\Part{X}$ that contains both the bottom- and top
elements, and therefore constitutes a hierarchical clustering of $X$.
\end{example}

Alternatively, a clustering of $X$ is an equivalence relation $\mscr{R} \in \Eqs{X}$, and a hierarchical
clustering is a chain in $\Eqs{X}$ containing both the bottom- and top elements of $\Eqs{X}$. A cluster
is, then, an equivalence class in $\quotient{X}{\mscr{R}}$.
We will refer to clusters as equivalence classes, clusters or blocks depending on the context, all terms being 
frequently used in clustering literature.
 
\subsubsection{Dendrograms}
For the remainder of the paper, let $\R_+$ denote the non-negative reals. We generally assume that
$\R_+$ is equipped with the usual total order $\le$.

Now, for a set $X$, let $\Part{X}$ be partially ordered by partition refinement, and let
$\theta : \R_+ \to \Part{X}$ be an order preserving map.
Consider the following list of possible properties of $\theta$:
\begin{description}
  \item[D$1$.] $\forall t \in \R_+ \, \exists \vareps > 0 \ \text{s.t.} \ \theta(t) = \theta(t + \vareps)$.
  \item[D$2$.] $\exists t_0 > 0 \ \text{s.t.} \ \theta(t_0) = \{X\}$, 
    the greatest element of $\Part{X}$, 
  \item[D$3$.] $\theta(0) = S(X)$, the least element of $\Part{X}$. 
\end{description}
 
If $\theta$ satisfies D$1$, then $\theta$ corresponds to what \citet{CarlssonMemoli2013} refers to as
a \emph{persistent set}. If $\theta$ satisfies D$1$ and D$2$, then $\theta$ is what \citet{Jardine1971}
refers to as a \emph{numerically stratified dendrogram}, and if $\theta$ also satisfies D$3$,
then Jardine and Sibson refer to $\theta$ as a \emph{definite numerically stratified dendrogram}.
Furthermore, the concept we have referred to as partial dendrograms
corresponds to $\theta$ satisfying D$1$ and D$3$, so a partial dendrogram is the same as a definite
persistent set.

It is the authors' impression that the current use of the term dendrogram in conjunction
to classical hierarchical clustering mainly covers what Jardine and Sibson call a definite numerically
stratified dendrogram. We thus land on the following definitions:

\begin{definition} \label{def:dendrogram}
  A \deft{dendrogram over $X$} is an order preserving map $\theta : \R_+ \to \Part{X}$ satisfying
  axioms D$1$, D$2$ and D$3$. If $\theta$ satisfies D$1$ and D$3$, we call $\theta$ a
  \deft{partial dendrogram over $X$}.
\end{definition}

We will use the term dendrogram to denote both the graphical and the functional representation.
If $\image{\theta}=\{\mcal{B}_i\}_{i=0}^n$, we assume that the enumeration is compatible
with the order relation on $\Part{X}$; in other words, that $\{\mcal{B}_i\}_{i=0}^n$ is a chain in $\Part{X}$.
We denote \deft{the family of all dendrograms over $X$} by~$\D(X)$, and
\deft{the family of all partial dendrograms over $X$} by~$\PD(X)$.

\subsubsection{Dissimilarity measures and ultrametrics}
A \deft{dissimilarity measure} on a set $X$ is a function $d:X \times X \to \R_+$, satisfying
\begin{description}
\item[$d1.$] $\forall x \in X \, : \, d(x,x) = 0$,
\item[$d2.$] $\forall x,y \in X \, : \, d(x,y) = d(y,x)$.
\end{description}
If $d$ additionally satisfies
\begin{description}[resume]
\item[$d3.$] $\forall x,y,z \in X \, : \, d(x,z) \le \max\{d(x,y), d(y,z)\}$,
\end{description}
we call $d$ an \deft{ultrametric} \citep{Rammal1986}.
The pair $(X,d)$ is correspondingly called a \deft{dissimilarity space}
or an \deft{ultrametric space}. The family of all dissimilarity measures over $X$ is
denoted by~$\mcal{M}(X)$, and the family of all ultrametrics by $\U(X)$.

\begin{example}[Ultrametric] \label{ex:ultrametric-inequality}
Property $d3$ is referred to as the \deft{ultrametric inequality}, and is a
strengthening of the usual triangle inequality. In an ultrametric space $(X,\ultra)$, \emph{every triple of 
points is arranged in an isosceles triangle:} Let $a,b,c \in X$,
and let the pair $a,b$ be of minimal distance such that $\ultra(a,b) \le \min \{\ultra(a,c),\ultra(b,c)\}$. 
The ultrametric inequality gives us
\[
\begingroup
\setlength{\arraycolsep}{2px}
\left.
\begin{array}{rcccl}
\ultra(a,c) &\le& \max \{ \ultra(a,b), \ultra(b,c) \} &=& \ultra(b,c) \\
\ultra(b,c) &\le& \max \{ \ultra(b,a), \ultra(a,c) \} &=& \ultra(a,c) 
\end{array}
\right\} \ \Leftrightarrow \ \ultra(a,c) = \ultra(b,c).
\endgroup
\]
\end{example}

Ultrametrics show up in many different contexts, such as $p$-Adic number theory \citep{Holly2001},
infinite trees \citep{Huges2004}, numerical taxonomy \citep{Sneath1973} 
and also within physics \citep{Rammal1986}, just to cite a few.
For hierarchical clustering, ultrametrics are relevant because the dendrograms over a set are
in a bijective relation to the ultrametrics over the same set \citep{CarlssonMemoli2010}.

We shall also need the following terms, which apply to any dissimilarity space: 
The \deft{diameter} of $(X,d)$ is given by the maximal inter-point distance:
\[
\diam(X,d) \ = \ \max \{ \, d(x,y) \, | \, x,y \in X \, \}.
\]
And the \deft{separation} of $(X,d)$ is the minimal inter point distance:
\[
\sep(X,d) \ = \ \min \{ \, d(x,y) \, | \, x,y \in X \land x \ne y \, \}.
\]

It is a well known fact that there exists an injective map from
dendrograms to ultrametrics \citep{Jardine1971}:
\[
\Psi_X : \D(X) \longto \U(X).
\]
In \citep{CarlssonMemoli2010} the map $\Psi_X$ is shown to be a bijection.
If $\theta \in \D(X)$, the map is defined as
\begin{equation} \label{eqn:Psi-X-def}
  \Psi_X(\theta)(x,y) \, = \, \min \{\, t \in \R_+ \, | \, \exists B \in \theta(t) \, : \, x,y \in B \, \}.
\end{equation}
That is, the ultrametric distance is the least real number $t$ for which $\theta$ maps to a
partition where $x$ and $y$ are in the same block. 
The minimisation is well defined due to Axiom~D$1$.
The ultrametric can be read from the diagrammatic representation of the dendrogram as
the minimum height you have to ascend to in order to traverse from one element
to the other following the paths in the tree. 

\subsubsection{Classical hierarchical clustering}
\label{section:classical-hc}
We first need to recall linkage functions.
Our definition follows the lines of \citet{CarlssonMemoli2010}:
\begin{definition} \label{def:linkage-function}
  Let $\pow{X}$ denote the power set of $X$.
  A \deft{linkage functions} on $X$ is a map
  \[
  \Link :  \pow{X} \times \pow{X} \times \M(X) \longto \R_+,
  \]
  so that for each partition $Q \in \Part{X}$ and dissimilarity measure $d \in \mcal{M}(X)$,
  the restriction $\Link|_{Q \times Q \times \{d\}}$ is a dissimilarity measure on $Q$.

  \medskip
  
  The classical linkage functions are defined as
  \begin{flalign*}
    \setlength{\arraycolsep}{5pt}
    \begin{array}{lclcl}
      \text{{\bf Single linkage}} &:&
      \SLink(p,q,d) & \!\!\!\! = \!\!\!\! &
      \min_{x \in p} \min_{y \in q} d(x,y),  \\[.9em]
      \text{{\bf Complete linkage}} &:&
      \CLink(p,q,d) & \!\!\!\! = \!\!\!\! &
      \max_{x \in p} \max_{y \in q} d(x,y),  \\[.3em]
      \text{{\bf Average linkage}} &:&
      \ALink(p,q,d) & \!\!\!\! = \!\!\!\! &
      \dfrac{\sum_{x \in p} \sum_{y \in q} d(x,y)}{\card{p} \cdot \card{q}}.
    \end{array} & &
  \end{flalign*}
\end{definition}

\begin{definition}[Classical $\HC$] \label{def:classical-HC}
Given a dissimilarity space $(X,d)$ and a linkage function $\Link$, if we follow the procedure outlined in
Section~\ref{section:HC-at-a-glance}, using $\Link$ as the ``notion of dissimilarity'',
the result is a chain
of partitions $\{Q_i\}_{i=1}^{\card{X}-1}$ together with the dissimilarities $\{\rho_i\}_{i=1}^{\card{X}-1}$
at which the partitions were formed.
The sequence of pairs $\mcal{Q} = \{(Q_i,\rho_i)\}_{i=1}^{\card{X}-1}$ corresponds uniquely to a
dendrogram $\theta_\mcal{Q}$ as follows:
\begin{equation} \label{eqn:dendrogram-from-chain}
  \theta_\mcal{Q}(x) = Q_{\max\{i \in \N \, | \, \rho_i \le x\}}.
\end{equation}
We define a \deft{classical hierarchical clustering of $(X,d)$ using $\Link$} to be a dendrogram
\[
\HC^\Link(X,d) \ = \ \theta_\mcal{Q}
\]
obtained through this procedure.
\end{definition}

\begin{remark} \label{remark:monotonicity}
  Notice that~\eqref{eqn:dendrogram-from-chain} maps $\{(Q_i,\rho_i)\}_{i=1}^{\card{X}-1}$ to a dendrogram
  if and only if
  \begin{equation} \label{eqn:lower-monotone}
    \sep(Q_i,\Link) \le \sep(Q_{i+1},\Link) \quad \itfor \ 0 \le i < \card{X}-1.
  \end{equation}
Otherwise, the $\rho_i$ will not make up a monotone sequence, and the resulting function $\theta_\mcal{Q}$
will not be an order preserving map. Although all of $\SLink$, $\ALink$ and $\CLink$ 
satisfy~\eqref{eqn:lower-monotone}, it is fully possible to define linkage functions that do not.
\end{remark} 

At any point during the clustering process, if we encounter a partition
$Q$ with two distinct pairs of elements $(p_1,q_1),(p_2,q_2) \in Q \times Q$ for which
\[
\Link(p_1,q_1,d) \, = \, \Link(p_2,q_2,d) \, = \, \sep(Q,\Link),
\]
we say that the two connections are \deft{tied}, since they are both eligible candidates for
the next merge.
It is well known that $\HC^\SLink$ is invariant with respect to the order of resolution of
ties \citep{Jardine1971}, a property referred to as being \deft{permutation invariant},
a characteristic shared by neither $\HC^\ALink$ nor $\HC^\CLink$.

\section{Optimised hierarchical clustering}
\label{section:optimised-hc}

In this section we devise a permutation invariant version of hierarchical clustering based on the classical
definition.
The key to permutation invariance is in dealing with tied connections.
If we consider the procedure for hierarchical clustering outlined in Section~\ref{section:HC-at-a-glance},
we can resolve tied connections by picking a random minimal dissimilarity pair. The way the
procedure is specified, this turns $\HC^\Link$ into a \emph{non-deterministic} algorithm;
it may produce different dendrograms for the same input in the presence of ties, depending on which tied
pair is selected. But more importantly, it is capable of producing \emph{any} dendrogram that can be
produced by \emph{any} tie resolution order:

\begin{definition}
Given a dissimilarity space $(X,d)$ and a linkage function $\Link$, 
let $\D^\Link(X,d)$ be \deft{the set of all possible outputs from $\HC^\Link(X,d)$}.
\end{definition}

A dissimilarity measure $d$ over a finite set $X$ can be described as 
an $\card{X} \times \card{X}$ real matrix $[d_{i,j}]$. Hence, given an ultrametric $\ultra \in \U(X)$
we can compute the pointwise difference
\begin{equation} \label{eqn:p-norm}
  \norm{\ultra - d}_p \ = \ \sqrt[\leftroot{0}\uproot{10}{\displaystyle p}]{ 
    \sum_{x,y \in X} \abs{\ultra(x,y) - d(x,y)}^p
  }.
\end{equation} 

We suggest the following definition, recalling the definition of $\Psi_X$~\eqref{eqn:Psi-X-def}:
\begin{definition} \label{def:opt-HC}
Given a dissimilarity space $(X,d)$ and a linkage function $\mcal{L}$,
\deft{the optimised hierarchical agglomerative
clustering over $(X,d)$ using $\mcal{L}$} is given by
\begin{equation} \label{eqn:opt-HC}
\HC_\opt^\mcal{L}(X,d) \ = \ \argmin_{\theta \in \D^\Link(X,d)} \norm{\Psi_X(\theta) - d}_p.
\end{equation}
\end{definition}

That is; among all dendrograms that can be generated by $\HC^\Link(X,d)$,
optimised hierarchical agglomerative clustering picks the dendrogram that is 
closest to the original dissimilarity measure.
In the tradition of ultrametric fitting, this is the right choice of candidate.

As $\D^\Link(X,d)$ contains all dendrograms generated over all possible
permutations of enumerations of $X$, the below theorem follows
directly from Definition~\ref{def:opt-HC}:

\begin{theorem} \label{thm:permutation-invariance}
$\HC_\opt^\Link$ is permutation invariant. That is, the order of enumeration of the elements of the
set $X$ does not affect the output from $\HC_\opt^\Link(X,d)$.
\end{theorem}
And since $\HC^\SLink$ is permutation invariant, we have $\big| \D^\SLink(X,d) \big| = 1$, yielding
\begin{theorem} \label{thm:SL-is-optimised}
$\HC_\opt^\SLink(X,d) = \HC^\SLink(X,d)$.
\end{theorem}

Since $\HC^\ALink$ and $\HC^\CLink$ are not permutation invariant, there is no
corresponding result in these cases.
For complete linkage, however, we have the following theorem. First, notice that due to the definition
of complete linkage (Definition~\ref{def:linkage-function}), 
if $\theta$ is a solution to $\HC_\opt^\CLink(X,d)$ and $\ultra = \Psi_X(\theta)$ is the corresponding
ultrametric, then
\[
\ultra(x,y) \ge d(x,y) \quad \forall x,y \in X.
\]
Hence, in the case of complete linkage we can reformulate~\eqref{eqn:opt-HC} as follows:
\begin{equation} \label{eqn:opt-HC-CL}
\HC_\opt^\CLink(X,d) \ = \ \argmin_{\theta \in \D^\CLink(X,d)} \norm{\Psi_X(\theta)}_p.
\end{equation}
To see why this is the case,
notice that if $u,u' \in \mcal{M}(X)$ and both $d \le u$ and $d \le u'$ pointwise, then
we can produce two non-negative functions $\delta,\delta'$ on $X \times X$ so that
$u = d + \delta$ and $u' = d + \delta'$. In particular, we have $u-d = \delta$,
from which we deduce
\begin{gather*}
  \norm{u-d}_p \le \norm{u'-d}
  \ \Leftrightarrow \
  \norm{\delta}_p \le \norm{\delta'}_p
  \ \Leftrightarrow \
  \norm{d + \delta}_p \le \norm{d + \delta'}_p
  \ \Leftrightarrow \
  \norm{u}_p \le \norm{u'}_p.
\end{gather*}

\begin{theorem} \label{thm:CL-is-NP-hard}
Solving $\HC_\opt^\CLink(X,d)$ is NP-hard.
\end{theorem}
\begin{proof}
Let $G=(V,E)$ be an undirected graph with vertices $V$ and edges $E \subs V \times V$.
Recall the \emph{clique} problem: Given a positive integer $K < |V|$, 
is there a clique in $G$ of size at least $K$? Equivalently: is there a set $V' \subs V$ with
$|V'| \ge K$ for which $V' \times V' \subs E$? This is a known NP-hard problem \citep{Karp1972}.

To reduce \emph{clique} to $\HC_\opt^\CLink$, define a dissimilarity measure on $V$ as follows:
\begin{equation} \label{eqn:clique-dissimilarity}
d(v,v') = 
\begin{cases}
  1 & \text{if $(v,v') \in E$},\\
  2 & \text{otherwise}.
\end{cases}
\end{equation}
Then $(V,d)$ is a dissimilarity space. Let $\theta$ be a solution of $\HC_\opt^\CLink(V,d)$,
and set $\ultrax{d} = \Psi_V(\theta)$.

An intrinsic property of $\CLink$ is that if two blocks $p,q \in Q_i$ are merged, then
\[
\forall v,v' \in p \cup q \ : \ d(v,v') \le \CLink(p,q,d).
\]
And since we have $d(v,v') = 1 \Leftrightarrow (v,v') \in E$, it means that 
for a subset $V' \subs V$, we have that
\begin{equation} \label{eqn:is-in-clique-condition}
\forall v,v' \lin V' \, : \, \ultrax{d}(v,v') = 1 
\ \Leftrightarrow \
\text{$V'$ is a clique in $G$}.
\end{equation}
It follows that a largest possible cluster at 
proximity level $1$ is a maximal clique in $G$.

We claim that minimising the norm is equivalent to
producing a maximal cluster at proximity level $1$:
Let $\ultrax{d}$ be the $\card{V} \times \card{V}$ distance matrix $[\ultrax{d}_{i,j}]$.
Due to the definition of $\CLink$, we have $\ultrax{d}(v,v') \in \{0,1,2\}$.
If $\theta(1) = \{V_i\}_{i=1}^s$, then these are exactly the blocks that are subsets of cliques,
so each $V_i$ contributes with $|V_i|(|V_i|-1)$ ones in $[\ultrax{d}_{i,j}]$. 

Having more ones reduces the norm of $\ultrax{d}$. 
Let $V_j$ be of maximal cardinality in~$\{V_i\}_{i=1}^s$.
Assume first that $V_j$ has at least two elements more than the next to largest block,
and let $|V_j|=P$.

Removing one element from $V_j$ reduces the number of ones in the dissimilarity matrix by
$P(P-1)-(P-1)(P-2)=2(P-1)$. Let the next to largest block have $Q$ elements. Transferring the element
to this block then increases the number of ones by $(Q+1)Q - Q(Q-1)=2Q$. Since $Q < P-1$,
this means that the total number of ones is reduced by moving an element from the largest block to any
of the smaller blocks. Hence, achieving the largest possible number of ones implies maximising the size of the 
largest block.

If now, $V_j$ only has one element more than the next to largest block, 
moving an element as above corresponds to keeping the number of ones.
Since each $V_i$ for $1 \le i \le s$ is a subset of a clique in $G$,
the maximal number of ones is achieved by producing a block $V_j$ that contains exactly a
maximal clique of $G$.

Therefore, if $\mcal{I}_{\{1\}}(x)$ is the indicator function for the set $\{1\}$,
the size of a maximal clique in $G$ can be computed as
\[
\max_{1 \le i \le \card{V}} \Big\{ 
  \sum_{j=1}^\card{V} \mcal{I}_{\{1\}}\big( \ultrax{d}_{i,j} \big)
\Big\},
\]
counting the maximal number of row-wise ones in $[\ultrax{d}_{i,j}]$ in $O(N^2)$ time.
We therefore conclude that $\HC_\opt^\CLink$ is NP-hard.
\spqed
\end{proof}

\boldbox{%
The computational hardness of \boldmath{$\HC_\opt^\CLink$} is directly connected to the presence of tied
connections: every encounter of \boldmath{$n$} tied connections leads to~\boldmath{$n!$} 
new candidate solutions.
}%
Since neither $\HC_\opt^\ALink$ is permutation invariant, the authors strongly believe that this is
also NP-hard, although that remains to be proven.

\bigskip

We cannot in general expect the mapping $\theta \mapsto \norm{\Psi_X(\theta) - d}_p$
to be injective, meaning that the answer to~\eqref{eqn:opt-HC} may not be unique.
Recall that $\pow{X}$ denotes the power set of $X$.
We shall consider $\HC_\opt^\Link(X,-)$ to be the function
\[
\HC_\opt^\Link(X,-) :\mcal{M}(X) \longto \pow{\D(X)},
\]
mapping a dissimilarity measure over $X$ to a set of dendrograms over $X$.

\subsection{Other permutation invariant solutions}
\citet{CarlssonMemoli2010} offer an alternative approach to permutation invariant hierarchical agglomerative
clustering. In their solution, when they face a set of tied connections, they merge all tied the pairs
in one operation, resulting in permutation invariance.

In the case of order preserving clustering, a family of tied connections can contain several mutually
exclusive merges due to the order relation. Using the method of \CarlssonMemoli{} leads to a problem
of figuring which blocks of tied connections to merge together, and in which combinations and order.
This leads to a combinatorial explosion of alternatives. The method we have suggested is utterly simple,
but it is designed to circumvent this very problem.

\section{Order preserving clustering}
\label{section:order-theory}
In this section, we determine what it means for an equivalence relation to be order preserving
with regards to a strict partial order, and establish precise conditions that are necessary and sufficient for 
a hierarchical agglomerative clustering algorithm to be order preserving.

\subsection{Order preserving equivalence relations}

Recalling the definition of a clustering (Definition~\ref{def:clustering}),
let $(X,<)$ be a strict poset. If $\mscr{R}$
is an equivalence relation on $X$ with quotient map $q:X \to \quotient{X}{\mscr{R}}$, we have
already established, in Section~\ref{section:motivating-use-case}, that we require
\[
\forall x,y \lin X \, : \, x<y \, \Rightarrow \, q(x) <' q(y).
\]

\boldbox{%
  That is, we are looking for a particular class of equivalence relations;
  namely those for which the quotient map is order preserving.%
}%

Given a strict poset $(X,E)$, there is a particular induced relation on the quotient
set $\quotient{X}{\mscr{R}}$ for any equivalence relation $\mscr{R} \in \Eqs{X}$ \citep[\S 3.1]{Blyth2005}:
\begin{definition} \label{def:induced-relation}
  Given a strict poset $(X,E)$ and an equivalence relation $\mscr{R} \in \Eqs{X}$,
  first define the relation $S_0$ on $X$ by
  \begin{equation} \label{eqn:induced-relation}
    ([a],[b]) \in S_0 \, \Leftrightarrow \,
    \exists x,y \in X \, : \, a \sim x \land b \sim y \land (x,y) \in E. 
  \end{equation}
  The transitive closure of $S_0$ is called
  \deft{the relation on $\quotient{X}{\mscr{R}}$ induced by $E$}.
  We denote this relation by~$S$.
\end{definition}

\newcommand{\triline}{{%
\arrow[dd,-,crossing over,shift left=0]%
\arrow[dd,-,shift left=1.5,crossing over]%
\arrow[dd,-,shift right=1.5,crossing over]}}%
\newcommand{\lessline}{{\arrow[rruu,->,crossing over]}}

\begin{example} \label{ex:fence}
An instructive illustration of what the relation $S_0$ looks like for a strict poset $(X,<)$ 
under the equivalence relation $\mscr{R}$ is that of an \deft{$\mscr{R}$-fence} \citep{Blyth2005},
or just fence, for short:
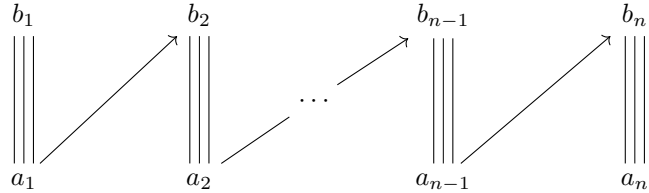
\begin{center}
\begin{tikzcd}
b_1 \triline & & b_2 \triline      &                      & b_{n-1} \triline  & & b_n \triline \\
             & &                   & \cdots \arrow[ru,->] &                   & &              \\
a_1 \lessline & & a_2 \arrow[ru,-] &                      & a_{n-1} \lessline & & a_n
\end{tikzcd}
\end{center}
Triple lines represent equivalences under $\mscr{R}$, 
and the arrows represent the order on \hbox{$(X,<)$}. The fence illustrates visually how one
can traverse from~$a_1$ to~$b_n$ along arrows and through equivalence classes in $\quotient{X}{\mscr{R}}$, 
and in that case we say that the fence \deft{links}~$b_1$ to~$a_n$.
\textbf{\boldmath{The induced relation $S$ has the property that $(a,b) \in S$
  if there exists an $\mscr{R}$-fence in $X$ linking $a$ to $b$.}}
\end{example}

Recall that a cycle in a relation $R$ is a sequence of pairs starting and ending
with the same element: $(a,b_1),(b_1,b_2),\ldots,(b_n,a)$.
The below theorem is an adaptation of \citep[Thm.3.1]{Blyth2005} to strict partial orders.
\begin{theorem} \label{thm:regular-equivalence-relation}
Let $(X,E)$ be a strict poset, $\mscr{R} \in \Eqs{X}$, and let $S$ be the relation on
$\quotient{X}{\mscr{R}}$ induced by $E$. Then the following statements are equivalent:
\begin{enumerate}
\item $S$ is a strict partial order on $\quotient{X}{\mscr{R}}$; \label{item:spo}
\item There are no cycles in $S_0$;       \label{item:acyclic}
\item $q_\mscr{R} : (X,E) \longto (\quotient{X}{\mscr{R}},S)$ is order preserving. \label{item:q-preserves}
\end{enumerate}
\end{theorem}
\begin{proof}
From the definition of strict posets, they contain no cycles, so $1 \Rightarrow 2$.
Since a non-cyclic set is irreflexive, and since $S$ is transitive by construction, $2 \Rightarrow 1$.

Let $q_\mscr{R}$ be order preserving. Notice that if $S_0$ is the set defined 
in~\eqref{eqn:induced-relation}, we have $S_0 = q_\mscr{R} \times q_\mscr{R}(E)$.
In particular, for all $x,y \in X$ for which $(x,y) \in E$, we have $([x],[y]) \in S_0$.
Assume that $S$ is not a strict order. Then there is a cycle in $S_0$;
that is there are $x,y \in X$ for which $(x,y) \in E$, 
but $([y],[x]) \in S_0$ also. This yields
\[
\exists a',b' \in X \, : \, a' \sim x \land b' \sim y \land (b',a') \in E.
\]
But, since $([x],[y]) \in S_0$, we also have
\[
\exists a,b \in X \, : \, a \sim x \land b \sim y \land (a,b) \in E.
\]
This yields $a \sim a'$ and $b \sim b'$, so we have 
\[
\big( q_\mscr{R}(a),q_\mscr{R}(b) \big) \in S_0
\ \land \
q_\mscr{R}(b) = q_\mscr{R}(b')
\ \land \
\big( q_\mscr{R}(b'),q_\mscr{R}(a') \big) \in S_0.
\]
But, since we have both $q_\mscr{R}(a)=q_\mscr{R}(a')$ and $(a,b) \in E$,
this contradicts the fact that $q_\mscr{R}$ is order preserving, so our assumption that both
$([x],[y])$ and $([y],[x])$ are elements of $S_0$ must be wrong. 
Hence, if $q_\mscr{R}$ is order preserving, there are no cycles in $S_0$, 
and $S$ is a strict partial order on $\quotient{X}{\mscr{R}}$.
This shows that $3 \Rightarrow 1$.

Finally, let $S$ be a strict partial order, and assume that $q_\mscr{R}$ is not order preserving.
Then, there exists $x,y \in X$ where $(x,y) \in E$ and for which at least one of $([x],[y]) \not \in S$
or $([y],[x]) \in S$ holds. Now, $([x],[y]) \in S$ by Definition~\ref{def:induced-relation}.
Therefore,~$([y],[x]) \in S$ implies that~$S$ has a cycle, contradicting the fact that~$S$ is
a strict partial order. 
\spqed
\end{proof}

\begin{definition}\label{def:regular_equivalence_relation}
  Let $(X,E)$ be a strict poset.
  An equivalence relation $\mscr{R} \in \Eqs{X}$ is \deft{regular} if 
  there exists an order on $\quotient{X}{\mscr{R}}$ for which the quotient map is order preserving.
  We denote \deft{the set of all regular equivalence relations} over an 
  ordered set~$(X,<)$ by~$\Eqs{X,<}$. Likewise, the family of all \deft{regular partitions} of 
  $(X,<)$ is denoted $\Part{X,<}$.
\end{definition}

In general, we will denote the induced order relation for a strict poset $(X,<)$
and a regular equivalence relation $\mscr{R} \in \Eqs{X,<}$ by $<'$.

\subsection{The structure of regular equivalence relations}
We now establish a sufficient and necessary condition for an agglomerative clustering algorithm
to be order preserving.
Recall that, if $A \subs X$, $\quotient{X}{A}$ denotes the quotient for which
the quotient map $q_A : X \to \quotient{X}{A}$ sends all of $A$ to a point, and is the identity
otherwise. 
That is, for every $x,y \in X$, we have
\[
q_A(x) = q_A(y) \ \Leftrightarrow \ x,y \in A.
\]

\begin{theorem} \label{thm:one-point-quotient}
If $A \subs X$ for a strict poset $(X,<)$, the quotient map $q_A : X \to \quotient{X}{A}$ is
order preserving if and only if $A$ is an antichain in $(X,<)$.
\end{theorem}
\begin{proof}
If $A$ is not an antichain, then $\quotient{X}{A}$ places comparable elements in the same equivalence
class, so $q_A$ is not order preserving.

Assume $A$ is an antichain. If $q_A$ is not order preserving, then there is a cycle 
in~$(\quotient{X}{A},<')$, and since we have only one non-singleton equivalence class, the
cycle must be on the form
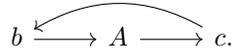
\begin{center}
\begin{tikzcd}
  b \arrow[r,->] & A \arrow[r,->] & c \arrow[ll,bend right=30,swap].
\end{tikzcd}
\end{center}
But this means we have $a,a' \in A$ for which $b < a$ and $a'<c$, but since $c < b$, this implies
$a'<a$, contradicting the fact that $A$ is an antichain.
\spqed
\end{proof}

Since a composition of order preserving maps is order preserving, this also applies to a composition
of quotient maps for a chain of regular equivalence relations $\mscr{R}_1 \subs \cdots \subs \mscr{R}_n$.
Combining this with Theorem~\ref{thm:one-point-quotient}, we have the following:

\boldfbox{
  A clustering of a strict poset will be order preserving if it can be produced as
  a sequence of pairwise merges of non-comparable elements.
}

We close the section with an observation 
about the family of all hierarchical clusterings over a strict poset:

\begin{theorem} \label{thm:meet-semilattice}
For a strict poset $(X,<)$, the set $\Part{X,<}$ of regular partitions over $(X,<)$
has $S(X)$ as its least element.
Unless $<$ is the empty order, there is no greatest element.
\end{theorem}
\begin{proof}
$S(X)$ is always a regular partition, so $S(X) \in \Part{X,<}$.
And since $S(X)$ is a refinement of every partition of $X$, $S(X)$ is
the least element of $\Part{X,<}$. 

If the order relation is not empty, then there are at least two elements that are comparable,
and, according to Theorem~\ref{thm:one-point-quotient}, they cannot be in the same equivalence class.
Hence, there is no greatest element.
\spqed
\end{proof}

The situation of Theorem~\ref{thm:meet-semilattice} is depicted in Figure~\ref{fig:meet-semilattice}, and
has already been discussed in Section~\ref{section:introducing-a-strict-order}:
In the case of tied connections that represent mutually exclusive merges,
choosing to merge one connection over the other may lead to
very different results.
We therefore need a strategy to select one of these
solutions over the others. This will be the main focus of Sections~\ref{section:partial-dendrograms}
and~\ref{section:HC}.

\section{Partial dendrograms} \label{section:partial-dendrograms}
In this section, we construct the embedding of partial dendrograms into ultrametrics.
Let an \deft{ordered dissimilarity space} be denoted by $(X,<,d)$. We generally assume that the order
relation is non-empty, meaning that there are comparable elements in $(X,<)$. 
Recall the partial dendrograms of Figure~\ref{fig:partial-dendrograms-intro}, and the
mathematical definition of a partial dendrogram in Definition~\ref{def:dendrogram}.
Partial dendrograms are clearly a generalisation of dendrograms. To distinguish between the
two, we will occasionally refer to the non-partial dendrograms as \deft{complete dendrograms}.

\medskip

For a partial dendrogram $\theta$, we  will write $\theta(\infty)$ to denote the maximal
partition in the image of $\theta$.
The only difference between a partial dendrogram and a complete dendrogram is that for a partial dendrogram
we do not require a greatest element in the image of $\theta$. 
However, since $\Part{X,<}$ is finite, a partial dendrogram $\theta \in \PD(X,<)$ is
\emph{eventually constant}; that is, there exists a positive real number $t_0$ for which 
\[
t \ge t_0 \ \Rightarrow \ \theta(t) = \theta(\infty).
\] 
We call the smallest such number the \deft{diameter} of $\theta$, formally given by
\[
\diam(\theta) \ = \ \min \{ x \in \R_+ \, | \, \theta(x) = \theta(\infty) \}.
\]

Looking at the partial dendrograms of Figure~\ref{fig:partial-dendrograms-intro},
each connected component in a partial dendrogram is a complete dendrogram over its leaf nodes.
Since complete dendrograms map to ultrametrics, each connected component 
gives rise to an ultrametric on the subset of $X$ constituted by the connected component's leaf nodes.
That is, if $\theta(\infty) = \{B_j\}_{j=1}^k$, and if $\theta_j$ is the complete dendrogram
over $B_j$ for $1 \le j \le k$, we can define the ultrametrics $\ultra_j = \Psi_{B_j}(\theta_j)$
so that $\left\{(B_j,\ultra_j)\right\}_{j=1}^k$ is a disjoint family of ultrametric
spaces, which union covers $X$.

\medskip

Now consider the following general result. 
\begin{lemma} \label{lemma:ultrametric-extension}
Given a family of bounded, disjoint ultrametric spaces $\{(X_j,d_j)\}_{j=1}^n$
together with a positive real number $K \ge \max_j\left\{\diam(X_j,d_j)\right\}$, the map
\[
d_\cup : \bigcup X_j \times \bigcup X_j \longto \R_+
\] 
given by
\[
  d_\cup(x,y) =
  \begin{cases}
    d_j(x,y) & \text{if $\exists j : x,y \lin X_j$}, \\
    K & \text{otherwise}
  \end{cases}
\]
is an ultrametric on $\bigcup_j X_j$.
\end{lemma}
\begin{proof}
To prove that the ultrametric inequality holds,
we start by showing that $d_{\cup_{1,2}}$ is an ultrametric on the restriction to the disjoint
union $X_1 \cup X_2$: Let $x,y \in X_1$ and $z \in X_2$, and choose a positive
$K \ge \max\{\diam(X_1,d_1),\diam(X_2,d_2)\}$. We now have
\begin{align*}
d_{\cup_{1,2}}(x,z) &= K & d_{\cup_{1,2}}(x,y) &= d_1(x,y) & d_{\cup_{1,2}}(y,z) &= K.
\end{align*}
This means that every triple of points are either already contained in an ultrametric space,
or they make up an isosceles triangle. In both cases, the ultrametric inequality holds,
according to the observation in Example~\ref{ex:ultrametric-inequality}.

By induction, we can now prove that $\big((X_1 \cup X_2) \cup X_3),d_{\cup_{1,2,3}} \big)$
is an ultrametric space, and so on, until all the $(X_j,d_j)$ are included.
\spqed
\end{proof}

Hence, for our partial dendrogram $\theta$ with $\theta(\infty)=\{B_j\}_{j=1}^k$ and subspace
ultrametrics $\{\ultra_j\}_{j=1}^k$,
pick a $K \ge \max_j\{\diam(B_j,\ultra_j)\}$, and define $\ultra_\theta : X \times X \to \R_+$ by
\begin{equation} \label{eqn:ultrametric-u}
\ultra_\theta(x,y) = 
\begin{cases}
\ultra_j(x,y) & \text{if $\exists j : x,y \lin B_j$}, \\
K & \text{otherwise}.
\end{cases}
\end{equation}
According to Lemma~\ref{lemma:ultrametric-extension}, equation \eqref{eqn:ultrametric-u} is an
ultrametric on $X$.

\begin{definition} \label{def:Ultra-eps}
Given an ordered space $(X,<)$ and a non-negative real number $\vareps$,
\deft{the ultrametric completion on $\vareps$} is the map
$\Ultra_\vareps : \PD(X,<) \longto \U(X)$ mapping
\[
\Ultra_\vareps : \theta \mapsto \ultra_\theta,
\]
where $\ultra_\theta$ is defined as in~\eqref{eqn:ultrametric-u}, setting $K = \diam(\theta) + \vareps$.
\end{definition}

\begin{example}
To illustrate how the ultrametric completion turns out in the case of the partial dendrograms
of Figure~\ref{fig:partial-dendrograms-intro}, we have the following figure:

\begin{center}
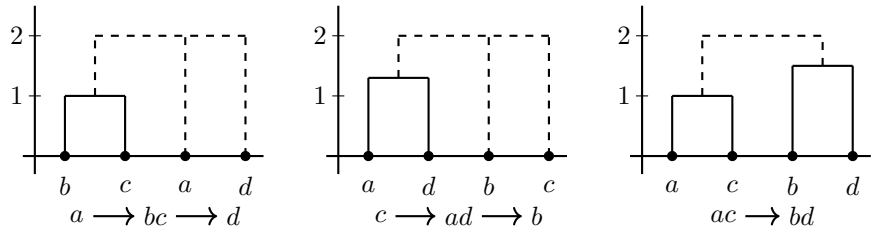

\vspace{5px}
\begin{tabular}{ccc}
  \input{partial_dendrogram_a_bc_d_completed} &
  \input{partial_dendrogram_c_ad_b_completed} &
  \input{partial_dendrogram_ac_bd_completed}   
\end{tabular}
\captionof{figure}{``Completed'' dendrograms corresponding to 
the partial dendrograms of Figure~\ref{fig:partial-dendrograms-intro}, using $K=2.0$. 
The completions are marked by the dashed lines.}
\label{fig:completed-dendrograms}
\end{center}
\end{example}

The above discussion serves to show that the construction is well defined.
Our next goal is two-fold. First, we wish to provide an (explicit) function from partial
dendrograms to dendrograms that realises this map. And second, we wish to establish
conditions for this function to be an embedding; that is, an injective map.
Injectivity is not strictly required for the theory to work, but it increases the discriminative
power of the theory. An example to the contrary is provided towards the end of the section.

\bigskip

We have the map $\Psi_X : \D(X) \longto \U(X)$ from~\eqref{eqn:Psi-X-def}, mapping dendrograms
to ultrametrics. We now seek a map $\kappa_\vareps : \PD(X,<) \longto \D(X)$
making the following diagram commute:
\begin{equation} \label{diag:kappa-commutes}
\begin{tikzcd}[column sep=5em,every label/.append style = {font = \normalsize}]
\D(X) \ar{r}{\Psi_X} & \U(X) \\[2.5em]
\PD(X,<) \ar[u,"\kappa_\vareps"] \ar{ru}[swap]{\Ultra_\vareps} & 
\end{tikzcd}.
\end{equation}
Seeing that $\kappa_\vareps$ must map partial dendrograms to complete dendrograms,
a quick glance at Figure~\ref{fig:completed-dendrograms} suggests the following definition:
\begin{equation*}
\kappa_\vareps(\theta)(x) \ = \ 
\begin{cases}
\theta(x) & \text{for $x < \diam(\theta) + \vareps$} \\
\{X\} & \text{otherwise}.
\end{cases}
\end{equation*}
It is straightforward to check that $\kappa_\vareps(\theta)$ is a complete dendrogram.

\begin{theorem} \label{thm:kappa-lifts}
$\Psi_X \circ \kappa_\vareps = \Ultra_\vareps$. That is; diagram~\eqref{diag:kappa-commutes} commutes.
\end{theorem}
\begin{proof}
Assume first that $\theta \in \PD(X,<)$ is a proper partial dendrogram, 
and that $\image{\theta} = \{\mcal{B}_i\}_{i=0}^n$.
Let the coarsest partition in the image of $\theta$ be given by $\mcal{B}_n = \{B_j\}_{j=1}^m$.
That is, each block $B_j$ corresponds to a connected component in the partial dendrogram.
Pick a block $B \in \mcal{B}_n$ and assume $x,y \in B$. 

If
\[
k = \min \{\, i \in \N \, | \, \exists B' \in \mcal{B}_i \, : \, B = B' \,\},
\]
then $\mcal{B}_k$ is the finest partition containing all of $B$ in one block.
Since $B \subs X$, the partitions
\[
\mcal{B}_i^B \ = \ \{ \, B \cap B' \, | \, B' \in \mcal{B}_i \, \} \quad \text{for $1 \le i \le k$}
\]
constitute a chain in $\Part{B}$ containing both $S(B)$ and $\{B\}$. 
Hence, we can construct a complete dendrogram over $B$ by defining
\begin{equation} \label{eqn:theta-B}
\theta_B(x) = \{\, B \cap B' \, | \, B' \in \theta(x) \, \}.
\end{equation}
This is exactly the complete dendrogram corresponding to the connected component of the 
tree over $X$ having the elements of $B$ as leaf nodes. 
By Definition~\ref{def:Ultra-eps}, 
\begin{align} \label{eqn:in-B-they-align}
  x,y \in B & \Rightarrow \ \Ultra_\vareps(\theta)(x,y) \, = \, \Psi_B(\theta_B)(x,y). 
\end{align}
Due to~\eqref{eqn:theta-B}, we have
\begin{multline*}  
  x,y \in B  \Rightarrow \ \left(
  \exists B \in \theta_B(x) \, : \, x,y \in B 
  \ \Leftrightarrow \ \exists B' \in \theta(x) \, : \, x,y \in B'
  \right) \\
   \Rightarrow \ 
  \min\{ \, t \in \R_+ \, | \, \exists B \in \theta_B(t) \, : \, x,y \in B \, \}
  \\ = \min\{ \, t \in \R_+ \, | \, \exists B' \in \theta(t) \, : \, x,y \in B' \, \}.
\end{multline*}
Hence, by the definition of $\Psi_X$ in~\eqref{eqn:Psi-X-def} we conclude that
\begin{align*}
  x,y \in B & \Rightarrow \ \Psi_B(\theta_B)(x,y) \, = \, (\Psi_X \circ \kappa_\vareps)(\theta)(x,y). 
\end{align*}
Combining this with~\eqref{eqn:in-B-they-align}, we get that whenever $x,y \in B$, 
we have $\Psi_X \circ \kappa_\vareps = \Ultra_\vareps$.

On the other side, let $x \in B_i$ and $y \in B_j$ with $i \ne j$. 
By definition, we have $\Ultra_\vareps(\theta)(x,y) = \diam(\theta) + \vareps$.
And, since there is no block in $\theta(\infty)$ containing both $x$ and $y$, 
we find that the minimal partition in $\image{\kappa_\vareps(\theta)}$ containing $x$ and $y$ in one block
is $\{X\}$. But this means that $\Psi_X(\kappa_\vareps(\theta))(x,y) = \diam(\theta) + \vareps$, 
so $\Psi_X \circ \kappa_\vareps = \Ultra_\vareps$ holds in this case too.

Finally, if $\theta$ is a complete dendrogram, we have $\kappa_\vareps(\theta) = \theta$, so 
$\Psi_X \circ \kappa_\vareps(\theta) = \Psi_X(\theta)$. But since $\theta(\infty) = \{X\}$,
it follows that $\Ultra_\vareps(\theta)$ maps exactly to the ultrametric over $X$ defined
by $\Psi_X(\theta)$. 
\spqed
\end{proof}

\begin{theorem} \label{thm:Ultra-is-injective}
  Let $(X,<)$ be a strict poset with a non-empty order relation. Then $\Ultra_\vareps$ is injective
  if $\vareps > 0$.
\end{theorem}
\begin{proof}
Since $\Ultra_\vareps = \Psi_K \circ \kappa_\vareps$ and $\Psi_X$ is a bijection, injectivity
follows if $\kappa_\vareps$ is injective. Assume that $\kappa_\vareps(\theta) = \kappa_\vareps(\theta')$.
Then, for every $x<\diam(\theta) + \vareps$, we have 
\[
\kappa_\vareps(\theta)(x) = \kappa_\vareps(\theta')(x)  \ \Leftrightarrow \ \theta(x) = \theta'(x).
\]
\spqed
\end{proof}

\begin{example} \label{ex:K-injectivity}
If $\vareps$ is not chosen to be strictly positive, the map $\Ultra_\vareps$
will not necessarily be injective. Consider the below dendrograms.
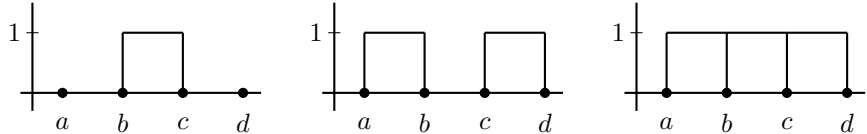
\begin{center}
\vspace{5px}
\begin{tabular}{c@{\hspace{.4cm}}c@{\hspace{.4cm}}c}
  \input{partial_dendrogram_counterex_01} &
  \input{partial_dendrogram_counterex_02} &
  \input{partial_dendrogram_counterex_completion}
\end{tabular} 
\end{center}
Both of the partial dendrograms are mapped to the same complete dendrogram (on the right)
for $\vareps = 0$.
This illustrates what we mean by \emph{reduced discriminative power} 
in the case of a non-injective completion.
Since the partial dendrograms exhibit distinctively different information, 
it is desirable that the methodology can distinguish them.
\end{example}

\section{Hierarchical clustering of ordered sets}
\label{section:HC}

We are now ready to embark on the specification of order preserving hierarchical clustering of ordered sets.
We do this by extending our notion of optimised hierarchical clustering from
Section~\ref{section:optimised-hc}.

\smallskip

Consider the following modification of classical hierarchical clustering.
The only difference is that for each iteration, we check that there are elements that
actually can be merged while preserving the order relation.
According to Theorem~\ref{thm:one-point-quotient}, this means merging a pair of non-comparable 
elements at each iteration. Recall that $S(X)$ denotes the singleton partition of $X$.

\bigskip

Let $(X,<,d)$ be given together with a linkage function~$\Link$.
\begin{enumerate}
\item Set $Q_0 = S(X)$, and endow $Q_0$ with the induced order relation $<_0$.
\item \label{step:ordered-pick} Among the pairs of non-comparable clusters, pick a pair of minimal
  dissimilarity  according to~$\Link$, and combine them into one cluster by taking their union. 
\item Endow the new clustering with the induced order relation.
\item If all elements of $X$ are in the same cluster, or if all clusters are comparable, we are done. 
  Otherwise, go to Step~\ref{step:ordered-pick} and continue.
\end{enumerate}

The procedure results in a chain of ordered partitions $\{(Q_i,<_i)\}_{i=0}^m$ together with the
dissimilarities $\{\rho_i\}_{i=0}^m$ at which the partitions where formed.
For an ordered set $(X,<)$, recall that non-comparability of $a,b \in X$ is denoted $a \noncmp b$.
Let the \deft{non-comparable separation of $(X,<,d)$}, be given by
\[
\sep_\noncmp(X,<,d) \ = \ \min_{x,y \in X} \{\, d(x,y) \, | \, x \ne y \land  x \noncmp y \,\}.
\] 
The reader may wish to compare the following lemma to Remark~\ref{remark:monotonicity}.
\begin{lemma}
  The sequence of pairs $\{(Q_i,\rho_i)\}_{i=0}^m$ produced by the above procedure
  maps to a partial dendrogram through application of~\eqref{eqn:dendrogram-from-chain} if and only if
  \[
  \sep_\noncmp(Q_i,<_i,\Link) \le \sep_\noncmp(Q_{i+1},<_{i+1},\Link) .
  \]
\end{lemma} 

Since the singleton partition $Q_0$ maps to a partial dendrogram, the algorithm will
produce a partial dendrogram for any ordered dissimilarity space, and since there can be at most
$|X|-1$ merges, the procedure always terminates. 

\medskip

As for classical hierarchical clustering, the procedure is non-deterministic in the sense that
given a set of tied pairs, we may pick a random pair for the next merge. Hence, the procedure is
capable of producing partial dendrograms for all possible tie resolution strategies:

\begin{definition} \label{def:all-partial-dendrograms}
Given an ordered dissimilarity space $(X,<,d)$ and a linkage function $\Link$, 
we write $\D^\Link(X,<,d)$ to denote the set of all possible outputs from the above procedure 
\end{definition}

{\flushleft The} set $\D^\Link(X,<,d)$ differs from $\D^\Link(X,d)$ in two important ways:
\begin{itemize}
\item $\D^\Link(X,<,d)$ contains partial dendrograms, not dendrograms.
\item The cardinality of $\D^\Link(X,<,d)$ is at least that of $\D^\Link(X,d)$, and often higher,
  due to mutually exclusive merges and the ``dead ends'' in $\Part{X,<}$ 
(see~Figure~\ref{fig:meet-semilattice}).
\end{itemize}
Even for single linkage we have $\big| \D^\SLink(X,<,d) \big| > 1$ if there
are mutually exclusive tied connections.

\medskip

In the spirit of optimised hierarchical clustering, we suggest the following definition,
employing the ultrametric completion $\Ultra_\vareps$ from Definition~\ref{def:Ultra-eps}:
\begin{definition} \label{def:ordered-HC}
  Given an ordered dissimilarity space $(X,<,d)$ together with a linkage function~$\Link$, 
  let $\vareps > 0$. An \deft{order preserving hierarchical agglomerative clustering using
    $\Link$ and $\vareps$} is given by
  \begin{equation} \label{eqn:ordered-HC}
    \HC_{\opt,\vareps}^{<\Link}(X,<,d) \ = \ 
    \argmin_{\theta \in \D^\Link(X,<,d)}\norm{\Ultra_\vareps(\theta) - d}_p.
  \end{equation}
\end{definition}

The next theorem shows that if we remove the order relation, then
optimised clustering and order preserving clustering coincide.
Keep in mind that a dissimilarity space
is an ordered dissimilarity space with an empty order relation; that is, $(X,d)=(X,\emptyset,d)$.
\begin{theorem} \label{thm:coincidence-on-empty-order}
  If the order relation is empty, then order preserving optimised hierarchical clustering and
  optimised hierarchical clustering coincide:
  \[
  \HC_{\opt,\vareps}^{<\Link}(X,\emptyset,d) = \HC_\opt^\Link(X,d).
  \]
\end{theorem}
\begin{proof}
First, notice that
\[
\forall \, (Q,<_Q) \in \Part{X,\emptyset} \ : \ \sep_\noncmp(Q,<_Q,\Link) = \sep(Q,\Link),
\]
where $<_Q$ denotes the (trivial) induced order. Hence, we have $\D^\Link(X,\emptyset,d) = \D^\Link(X,d)$.
Since $\Ultra_\vareps|_{\D(X)} = \Psi_X$, the result follows.
\spqed
\end{proof}

\subsection{On the choice of $\vareps$}
\label{section:on-vareps}

In $\HC_{\opt,\vareps}^{<\Link}(X,<,d)$ we identify the elements from $\D^\Link(X,<,d)$
that are closest to the dissimilarity measure~$d$ when measured in the $p$-norm.
The injectivity of $\Ultra_\vareps$ induces a relation $\pre_{d,\vareps}$ on $\PD(X,<)$ defined by
\[
\theta \pre_{d,\vareps} \theta' 
\ \Leftrightarrow \ 
\norm{\Ultra_\vareps(\theta) - d}_p \le \norm{\Ultra_\vareps(\theta') -d}_p,
\]
and the optimisation finds the minimal elements under this order.

The choice of $\vareps$ may affect the ordering of dendrograms under $\pre_{d,\vareps}$.
We show this by providing an alternative formula for $\norm{\ultra - d}_p$ 
that better expresses the effect of the choice of~$\vareps$.
Assume $\theta$ is a partial dendrogram over $(X,<)$ with $\theta(\infty) = \{B_i\}_{i=1}^m$,
and let $\Ultra_\vareps(\theta) = \ultra$.
We split the sum for computing $\norm{\ultra - d}_p$ in two:
the intra-block differences and the inter-block differences. The \deft{intra-block differences} are
independent of $\vareps$, and are given by
\begin{equation} \label{eqn:intra-block-differences}
\alpha \ = \ \sum_{i=1}^m \sum_{x,y \in B_i} \abs{\ultra(x,y) - d(x,y)}^p.
\end{equation}
On the other hand, the \deft{inter-block differences} are dependent on $\vareps$, and can be computed as
\begin{equation} \label{eqn:inter-block-pairs}
  \beta_\vareps \ = \ \sum_{\underset{i \ne j}{(x,y) \in B_i \times B_j}}
  \abs{\diam(\theta) + \vareps - d(x,y)}^p.
\end{equation}
This yields
$\norm{\ultra - d}_p = \sqrt[\leftroot{0}\uproot{3}{\scriptstyle p}]{\alpha + \beta_\vareps}$.
If we think of $\ultra$ as an approximation of $d$, and saying that $\card{X}=N$,
the mean $p$-th error of this approximation can be expressed as a function of $\vareps$:
\[
E_d(\vareps|\theta,p) \ = \ \frac{1}{N} \norm{u-d}_p^p \ = \
\frac{\alpha}{N} \ + \ \frac{1}{N}\sum_{\overset{(x,y) \in B_i \times B_j}{i \ne j}}
\abs{\diam(\theta) + \vareps - d(x,y)}^p.
\]

From the formula for $E_d(\vareps|\theta,p)$, we see that
when $\vareps$ becomes large, the inter-block differences dominate the approximation error. 
For increasing $\vareps$, having low error eventually equals having few inter-block pairs. 
Alternatively: the intra-block differences have insignificant influence on the approximation error
for large $\vareps$.
This means that as $\vareps$ increases beyond $\diam(X,d)$, the
partial dendrograms close to $d$ will be those that have a low number of inter-block pairs,
regardless of the quality of the intra-block ultrametric fit.
From the standpoint of ultrametric fitting, this is intuitively wrong.
Also, large $\vareps$ will lead to clusterings where as many elements as possible are placed in one
large cluster, since this is the most effective method for reducing the number of inter-block pairs.

On the other side, a low value of $\vareps$ will move the weight towards optimising the 
intra-block ultrametric fit. Since the inter-block distances are all set to $\diam(\theta) + \vareps$,
it is in the intra-block ultrametric fit we can make a difference in the optimisation. This will also
reduce the bias of cluster size as a function of $\vareps$.

In all, it is the authors' opinion that this points towards selecting a low value for $\vareps$.
In the process of choosing, we have the following result at our aid:

\begin{theorem} \label{theorem:vareps-0}
For any finite ordered dissimilarity space $(X,<,d)$ and linkage function $\Link$, there exists
an $\vareps_0 > 0$ for which
\[ 
\vareps,\vareps' \in (0, \vareps_0)
\ \Rightarrow \ 
\big(\D^\Link(X,<,d),\pre_{d,\vareps}) \approx \big(\D^\Link(X,<,d),\pre_{d,\vareps'}).
\]
That is; all $\vareps \in (0,\vareps_0)$ induce the same order on the partial dendrograms.
\end{theorem}
\begin{proof}
  Since $X$ is finite, $\D^\Link(X,<,d)$ is also finite.
  And according to $E_d(\vareps|\theta,p)$, if the cardinality of $\D^\Link(X,<,d)$ is $n$, 
  there are at most $pn$ positive values of $\vareps$ that are distinct
  global minima of partial dendrograms in $\D^\Link(X,<,d)$. But this means there
  is a finite set of $\vareps$ for which the order on
  $(\D^\Link(X,<),\pre_{\vareps,p})$ changes. And since all these values are strictly positive,
  they have a strictly positive lower bound.
\spqed
\end{proof}

Since the value of $\vareps_0$ depends on $D^\Link(X,<,d)$, it is non-trivial to compute.
For practical applications, we recommend to choose a very small positive number for~$\vareps$, but not
so small that it becomes zero due to floating point rounding when added to the diameter of the
partial dendrograms.

\subsection{Idempotency of $\HC_{\opt,\vareps}^{<\Link}$} \label{section:idempotency}
A detailed axiomatic analysis along the lines of for example \citet{Ackerman2016} is 
beyond the scope of this paper, and is considered for future work.
We still include a proof of idempotency of $\HC_{\opt,\vareps}^{<\Link}$, since this is
an essential property of classical hierarchical clustering.

Idempotency of hierarchical clustering necessarily depends on the linkage function.
We introduce the following concept, that allows us to prove this property for a range of linkage functions:
We say that $\Link$ is a \deft{convex linkage function} if we always have
\[
\SLink(p,q,d) \le \Link(p,q,d) \le \CLink(p,q,d).
\]
Notice that if $\ultra$ is an ultrametric on $X$, the ultrametric inequality yields
\[
\ultra(a,b) = \sep(X,\ultra) \ \Rightarrow \ \forall c \in X \ : \ \ultra(a,c) = \ultra(b,c),
\]
so if $\Link$ is a convex linkage function and $\ultra(a,b) = \sep(X,\ultra)$, we have
\[
\Link(\{a,b\},\{c\}) = \Link(\{a\},\{c\}) = \Link(\{b\},\{c\}) \quad \forall c \ne a,b.
\]
This is to say that a convex linkage function preserves the structure of the original ultrametric
when minimal dissimilarity elements are merged.
As a result, for any $\ultra \in \U(X)$, the set $\D^\Link(X,\ultra)$ contains exactly one element,
namely the dendrogram corresponding to the ultrametric, which is why classical hierarchical
clustering is idempotent.




For ordered spaces, the case is different. It is easy to construct an ordered ultrametric
space $(X,<,\ultra)$ for which $\ultra(a,b) = \sep(X,\ultra)$ and $a<b$, in which case the ultrametric
cannot be reproduced. Hence, all of $\U(X)$ cannot be fixed points under 
$\Ultra_\vareps \circ \HC_{\opt,\vareps}^{<\Link}(X,<,-)$, but the mapping is still idempotent:
\begin{theorem}[Idempotency]
  For an ordered dissimilarity space $(X,<,d)$ and a convex linkage function~$\Link$, we have
  $\theta \in \HC_{\opt,\vareps}^{<\Link}(X,<,d) \ \Rightarrow \
  \HC_{\opt,\vareps}^{<\Link}\left(X,<,\Ultra_\vareps(\theta)\right) = \{\theta\}$.
\end{theorem} 
\begin{proof}
  Let $\theta(\infty)=\{B_i\}_{i=1}^m$. Then each $B_i$ is an antichain in $(X,<)$, so we have
  \[
  \forall x,y \in B_i \ : \ \sep(B_i,\ultra|_{B_i}) \, = \, 
  \sep_\noncmp(B_i,\ultra|_{B_i}) \quad \for \ 1 \le i \le m.
  \]
  Since $\vareps > 0$, we also have
  \[
  x,y \in B_i \ \Rightarrow \ \ultra(x,y) < \diam(X,\ultra) \quad \for \ 1 \le i \le m.
  \]
  And, lastly, since every pair of comparable elements are in pairwise different blocks,
  we have
  \[
  x<y \lor y<x \ \Rightarrow \ \ultra(x,y) = \diam(X,\ultra).
  \] 
  Now, since $\Link$ is convex, based on the discussion
  preceding the theorem, the intra-block structure of every block
  will be preserved. And, since every inter-block dissimilarity is accompanied by comparability
  across blocks, the procedure for generation of $\D^\Link\!\left(X,<,\Ultra_\vareps(\theta)\right)$
  will exactly reproduce the intra block structure of all blocks and then halt.
  Hence, $\D^\Link\!\left(X,<,\Ultra_\vareps(\theta)\right) = \{\theta\}$.
\spqed
\end{proof}

\section{Polynomial time approximation}
\label{section:approximation}
In the absence of an efficient algorithm for $\HC_{\opt,\vareps}^{< \Link}$, this section provides a polynomial
time approximation scheme. The efficacy as approximation is demonstrated in
Section~\ref{section:approximation-demo}, and a demonstration on real world data is given
in Section~\ref{section:demo}.

\smallskip

Recall the set $\D^\Link(X,<,d)$ of partial dendrograms over $(X,<,d)$ from
Definition~\ref{def:all-partial-dendrograms}. The algorithm for producing a random
element of $\D^\Link(X,<,d)$ is described at the beginning of Section~\ref{section:HC};
the key is to pick a random pair for merging whenever we encounter a set of tied connections.

\smallskip

The approximation model is deceivingly simple; we generate a set of random partial dendrograms,
and choose the one with the best ultrametric fit.
\begin{definition} \label{def:approximation}
  Let $(X,<,d)$ be given, and let $N$ be a positive integer. For any random selection of $N$
  partial dendrograms $\{\theta_i\}_i$ from $\D^\Link(X,<,d)$,
  an \deft{$N$-fold approximation of $\HC_{\opt,\vareps}^{< \Link}(X,<,d)$} is a partial
  dendrogram $\theta \in \{\theta_i\}_i$ minimising $\norm{\Ultra_\vareps(\theta) - d}_p$.
  We denote the $N$-fold approximation scheme by $\HC^{<\Link}_{N,\vareps}$.
\end{definition}

\subsection{Running time complexity}
Assume that $\card{X}=n$.
In the worst case, we may have to check $n \choose 2$ pairs to find one that is
not comparable, and the test for $a \noncmp b$ has complexity $O(n^2)$, leading to
a complexity of $O(n^4)$ of finding a mergeable pair. Since there are up to $n-1$
merges, the worst case estimate of the running time complexity for producing one element
in $\D^\Link(X,<,d)$ is~$O(n^5)$.

\medskip

A part of this estimate is the number of comparability tests we have to perform
in order to find a mergeable pair. For a sparse order relation, we may have to test significantly
less than $n \choose 2$ pairs before finding a mergeable pair: if $K$ is the expected number of
test we have to do, the expected complexity of finding a mergeable pair becomes
$O(K n^2)$. This yields a total expected algorithmic complexity
of $O(K n^3)$. If the order relation is empty, we have $K=1$, and the complexity of producing a
dendrogram becomes $ O(n^3)$, which is the running time complexity of classical hierarchical clustering.
Hence, if the order relation is sparse, we can generally expect the algorithm to
execute significantly faster than the worst case estimate.

\medskip

When producing an $N$-fold approximation, the $N$ random partial dendrograms can be generated in
parallel, reducing the computational time of the approximation. For the required number of dendrograms to
obtain a good approximation, please see Section~\ref{section:approximation-demo}.

\section{Demonstration of approximation efficacy on randomly generated data}
\label{section:approximation-demo}
The purpose of the demonstration is to check to which degree the approximation reproduces
the order preserving clusterings of $\HC^{<\Link}_{\opt,\vareps}$.
We start by describing the random data model and the quality measures we use in assessing the efficacy of
the approximation, before presenting the experimental setup and the results.

\subsection{Random ordered dissimilarity spaces}
To test the correctness and convergence ratio of the approximation scheme, we employ randomly
generated ordered dissimilarity spaces. The random model consists of two parts: the random partial order and
the random dissimilarity measure.

\subsubsection{Random partial order}
A partial order is equivalent to a transitively closed directed acyclic graph, so we can use any random
model for directed acyclic graphs to generate random partial orders.
We choose to use the classical \ER{} random graph model \citep{Bollobas2001}.
Recall that a directed acyclic graph on $n$ vertices is a binary $n \times n$ adjacency matrix that is
\emph{permutation similar} to a strictly upper triangular matrix; that is, there exists a permutation
that, when applied to both the rows and the columns of one matrix, transforms it into the other.
Let this family of $n \times n$ matrices be denoted by $\RO(n)$.
For a number $p \in [0,1]$, the sub-family $\RO(n,p) \subs \RO(n)$ is defined as
follows: for $A \in \RO(n)$, let $A'$ be strictly upper triangular and
permutation similar to $A$. Then
each entry above the diagonal of $A'$ is $1$ with probability $p$.
The sought partial order is the transitive closure of this graph; we denote the corresponding set of
transitively closed directed acyclic graphs by~$\overline{\RO}(n,p)$.

\subsubsection{Random dissimilarity measure}
If $\card{X}=n$, a dissimilarity measure over $X$ with no tied connections consists of $n \choose 2$
distinct values. Hence, any permutation of the sequence $\{1,\ldots,{n \choose 2}\}$ is a non-tied
random dissimilarity measure over $X$.

To generate tied connections, let $t \ge 1$ be the \deft{expected number of ties per level}. That is,
for each unique value in the dissimilarity measure, that value is expected to have multiplicity~$t$.
In the case where $t$ does not divide $n \choose 2$, we resolve this by setting the multiplicity of
the largest dissimilarity to $\left({n \choose 2} \! \mod t \right)$.

We write $\RD(n,t)$ to denote the family of random dissimilarity measures over sets of $n$ elements
with an expected number of $t$ ties per level.

\begin{definition} \label{def:random-space}
  Given positive integers $n$ and $t$ together with $p \in [0,1]$, the family of
  \deft{random ordered dissimilarity spaces generated by $(n,p,t)$} is given by
  \[
  \RODS(n,p,t) \ = \ \overline{\RO}(n,p) \times \RD(n,t).
  \]
\end{definition}

\subsection{Measures of cluster quality}
In the demonstration, we start by generating a random ordered dissimilarity space. We then run the
optimal clustering method on the space, finding the optimal order preserving hierarchical
clustering. Finally, we run the approximation scheme on the space and study to which degree the approximation
manages to reproduce the optimal hierarchical clustering. For this, we need a quantitative measure of
clustering quality relative a known optimum.

A large body of literature exists on the topic of comparing clusterings
(see for instance \citep{VinhEppsBailey2010} for a brief review).
We have landed on the rather popular \emph{adjusted Rand index} \citep{HubertArabie1985}
to measure the ability of the approximation in finding a decent partition, comparing against the
optimal result.
 
Less work is done on this type of comparison for partial orders and directed acyclic graphs. We suggest
to use a modified version of the adjusted Rand index for this purpose too, based on an adaptation of
the Rand index used for network analysis \citep{Hoffman2015}. For an introduction to the Rand index, and also
to some of the versions of the adjusted Rand index, see \citep{Rand1971,HubertArabie1985,GatesAhn2017}.

\subsubsection{Adjusted Rand index for partition quality}
The Rand index compares two clusterings by computing the percentage of corresponding decisions made in
forming the clusterings; that is, counting whether pairs of elements are placed together in both
clusterings or apart in both clusterings.
An adjusted Rand index reports in the range $(-\infty,1]$, where zero is equivalent to a random draw,
and anything above zero is better than chance.
We use the adjusted Rand index (\mari) to compute the efficacy of the approximation in finding a
partition close to a given planted partition.
This corresponds to what \citet{GatesAhn2017} refers to as a \emph{one sided} Rand index, since one of
the partitions are given, whereas the other is drawn from some distribution. In the below
demonstration, we assume that the approximating partition is drawn from the set of all
partitions over $X$ under the uniform distribution.

\subsubsection{Adjusted Rand index for induced order relations}
When comparing induced orders on partitions over a set, unless the partitions coincide,
it is not obvious which blocks in one partition correspond to which blocks in the other.
To overcome this problem, we base our measurements on the base space projection:
\begin{definition} \label{def:base-space-projection}
  For an ordered set $(X,E)$ and a partition $Q$ of $X$ with induced order $E'$,
  \deft{the base space projection of $(Q,E')$ onto $X$} is the order relation $E_Q$ on $X$ defined as
  \[
  (x,y) \in E_Q \ \Leftrightarrow \ ([x],[y]) \in E'.
  \]
\end{definition}

{\flushleft This} allows us to compare the induced orders in terms of different orders on $X$.
Notice that if the induced order $E'$ is a [strict] partial order on $Q$, then $E_Q$ is a
[strict] partial order on $X$.

\smallskip

\citet{Hoffman2015} demonstrate that the adjusted Rand index can be used to detect missing links in networks
by computing the similarity of edge sets.
The concept relies on the fact that a network link and a link in an equivalence relation are not
that different: Both networks and equivalence relations are special classes of relations, and the Rand
index simply counts the number of coincidences and mismatches between two relation sets.
While \citet{Hoffman2015} uses the \mari{} to compare elements within a network, we use the same
method to compare across networks.

Let $A$ and $B$ be the adjacency matrices of two base space projections, and let $A_i$ denote the $i$-th
row of $A$, and likewise for $B_i$. If $\la a,b \ra$ is the inner product of $a$ and $b$, we define
\begin{equation*} 
  \begin{array}{ll}
    a_i = \la A_i, B_i \ra     & \quad c_i = \la A_i, 1 - B_i \ra \\
    b_i = \la 1 - A_i, B_i \ra & \quad d_i = \la 1 - A_i, 1 - B_i \ra.
  \end{array}
\end{equation*}
Here, $a_i$ is the number of common direct descendants of $i$ in both relations, $b_i$ is the number of
descendants of $i$ found in $A$ but not in $B$, $c_i$ is the number of descendants of $i$ in $B$ but not
in $A$, while $d_i$ counts the common non-descendants of $i$ in the two relations.
Using this, we can compute the \deft{element wise adjusted order Rand index}
\[
\oari_{i} = \frac{2(a_i d_i - b_i c_i)}{(a_i+b_i)(b_i+d_i)+(a_i+c_i)(c_i+d_i)}
\qquad \text{for $1 \le i \le n$},
\]
measuring the element wise order correlation between the base space projections in the
Hubert-Arabie adjusted Rand index \citep{HubertArabie1985,Warrens2008}\footnote{%
This particular formulation of the adjusted Rand index relies on the networks having known and fixed labels,
so that we know which vertices map to which vertices \citep{Warrens2008}, which indeed holds for the base
space projections of two different induced order relations.}.
Notice that we compare the $i$-th row in $A$ to the $i$-row in $B$ since these rows correspond to
the projections' respective descentand relations for the $i$-th element in $X$.
In \citep{Hoffman2015}, the above index is computed for each element pair \emph{within} the network
to produce the intra-network similarity coefficient.

Since we are interested in the overall match, we choose to report on the mean value,
defining the \deft{adjusted order Rand index for $A$ and $B$} as
\begin{equation*}
  \oari(A,B)
  \ = \ \frac{1}{n} \sum_{i=1}^n \oari_i.
\end{equation*}

\subsubsection{Normalised ultrametric fit}
A natural choice of quality measure is to report the \deft{ultrametric fit}
$\norm{\Ultra_\vareps(\theta) - d}_p$ of the obtained partial dendrogram $\theta$,
especially if we can compare it to the ultrametric fit of the optimal solution.
The scale of the ultrametric fit depends heavily on both the size of the space and the order of the norm,
so we choose to normalise.
Also, we invert the normalised value, so that the optimal fit has a value of $1$, and a worst possible fit
has value $0$. This makes it easy to compare the convergence of the ultrametric fit to the convergence of the
\mari{} and \moari{}.
\begin{definition}
  Given a set of partial dendrograms $\{\theta_i\}$ over $(X,<,d)$,
  let their respective ultrametric fits be given by $\delta_i = \norm{\Ultra_\vareps(\theta_i) - d}_p$.
  The \deft{normalised ultrametric fit} are the corresponding values
  \[
  \hat \delta_i = 1 - \frac{\delta_i - \min_i \{ \delta_i \}}{ \max_i \{ \delta_i\} - \min_i \{ \delta_i\} }.
  \]  
  In the presence of a reference solution, we substitute $\min_i \{ \delta_i \}$ with the
    ultrametric fit of the reference.
\end{definition}

\subsubsection{Ultrametric fit relative the optimal ultrametric}
\label{section:ari-vs-fit}
The reference partition can be reached through
different sequences of merges, and neither $\ALink$ nor $\CLink$ are invariant in this respect.
Neither \mari{}, \moari{} nor ultrametric fit captures the match between the optimal hierarchy
and the approximated hierarchy.
We therefore also include plots of the difference between the optimal ultrametric $\ultra_\opt$
and the approximated ultrametric $\ultra_{N,\vareps}$. Since both ultrametrics are equivalent to their
respective hierarchies, the magnitude $\norm{\ultra_\opt - \ultra_{N,\vareps}}_p$ can be interpreted
as a measure of difference in hierarchies. In the below plots, this is reported as $opt. fit$.
As for the ultrametric fit, we normalise and invert the values for easy comparison.

\subsection{Demonstration on randomly generated data}
The experiments in the demonstration split in two. First, we demonstrate the efficacy of the
approximation relative a known optimal solution, to see to which degree $\HC^{<\Link}_{N,\vareps}$
manages to approximate $\HC^{<\Link}_{\opt,\vareps}$.
Second, we study the convergence rate of the ultrametric fit for larger spaces with much
larger numbers of tied connections; spaces for which the optimal algorithm does not terminate within
any reasonable time.

\medskip

For each parameter combination in Table~\ref{table:params-ref-solution}, a set of $30$ random ordered
dissimilarity spaces are generated. For each space, $100$ approximations are generated
according to the prescribed procedure.
We then bootstrap the approximations to generate $N$-fold approximations for different $N$.

\begin{table}[hbpt]
  \begin{center}
    \begin{tabular}{l|l|l|l|l|c}
      & \multicolumn{1}{c|}{$n$} &
      \multicolumn{1}{c|}{$\Link$} &
      \multicolumn{1}{c|}{link probability ($p$)} &
      \multicolumn{1}{c|}{expected ties ($t$)} & reference   \\
      \hline
      Figure~\ref{fig:efficacy-varying-p} &
      $200$ & $\SLink$, $\ALink$, $\CLink$ & $0.01,0.02,0.05$ & $5$ & yes \\
      Figure~\ref{fig:efficacy-varying-t} &
      $200$ & $\SLink$, $\ALink$, $\CLink$ & $0.05$ & $3,7$ & yes \\
      Figure~\ref{fig:approximations-p01} &
      $500$ & $\SLink$, $\ALink$, $\CLink$ & $0.01$ & $10,50,100$ & no \\
      Figure~\ref{fig:approximations-p05} &
      $500$ & $\SLink$, $\ALink$, $\CLink$ & $0.05$ & $50,100$ & no \\
      Figure~\ref{fig:approximations-p10} &
      $500$ & $\SLink$, $\ALink$, $\CLink$ & $0.10$ & $100$ & no 
    \end{tabular}
    \caption{Parameter settings for the demonstrations. The right-most column indicates whether
      the reference clustering is available or not. The left-most column refers to the figure wherein
      the outcome of the corresponding experiment is presented.
      The parameters have been chosen to illustrate how the algorithm behaviour changes with changing
      expected number of ties, changing link probability in the random partial order, and
      choice of linkage function.}
    \label{table:params-ref-solution}
  \end{center}
\end{table}

We present the results in terms of convergence plots, showing the efficacy of the approximation as a
function of the sample size~$N$.
For the results where a reference solution is available, the plots contain four curves:

{\flushleft
\begin{tabular}{l@{ - }p{10cm}}
$\E(\ari)$ & The expected adjusted Rand index of the approximated partition. \\
$\E(\oari)$ & The expected adjusted Rand index of the approximated induced order.  \\
$norm. fit$ & The mean of the normalised fit. \\
$opt. fit$ & The mean of the normalised difference between the approximated ultrametric and
  the optimal ultrametric.
\end{tabular}}

\bigskip

For the results where no reference solution is available, we present the distribution of
the normalised fit.

\bigskip

The results are presented in Figures~\ref{fig:efficacy-varying-p}, \ref{fig:efficacy-varying-t}, \ref{fig:approximations-p01}
and~\ref{fig:approximations-p10} on pages~\pageref{fig:efficacy-varying-p}, \pageref{fig:efficacy-varying-t}, \pageref{fig:approximations-p01}
and~\pageref{fig:approximations-p10}, respectively. The parameter settings corresponding to the
figures are given in Table~\ref{table:params-ref-solution} for easy reference,
and are also repeated in the figure text.

\medskip

As we can see from the below results, the approximation generally performs very well.
We also see that a large expected number of tied connections requires larger sample size for a good
approximation, while a more dense order relation
(higher value of $p$) seems to require a smaller sample compared to a more sparse relation.
We also see that there is a seemingly strong correlation between the ultrametric fit of the approximation
and the similarity between the approximation ultrametric and the optimal ultrametric.

Regarding choice of linkage function, the approximation only requires small samples for both $\SLink$ and
$\ALink$, while $\CLink$ requires larger samples for larger numbers of tied connections.

\begin{figure}
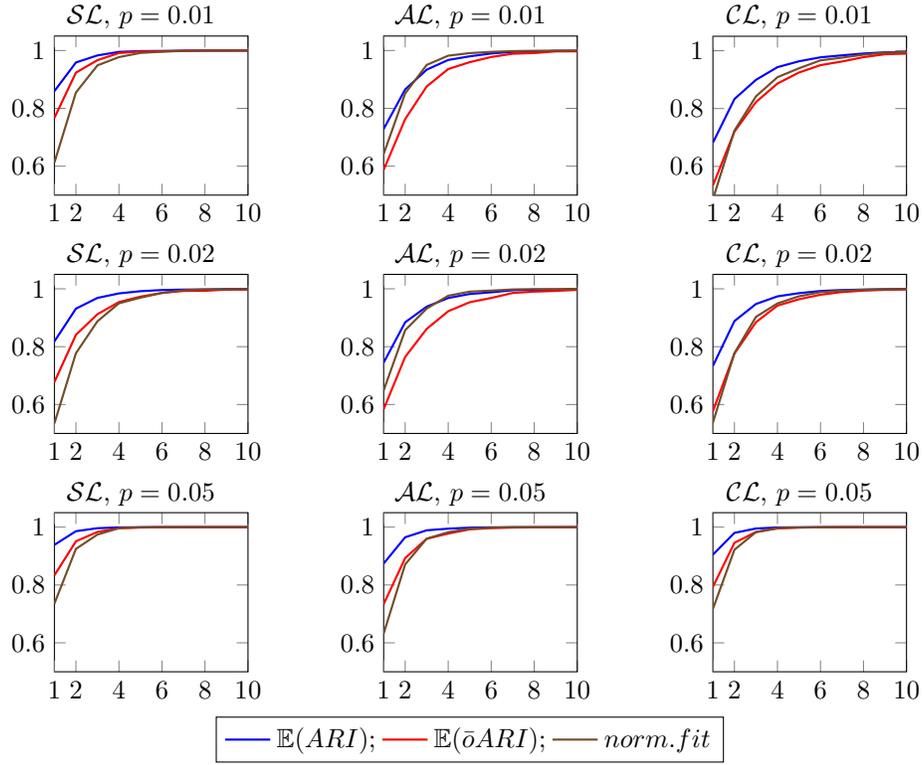

  \begin{center}
    \begin{tabular}{c@{$\qquad$}c@{$\qquad$}c}
      \makecell{
        $\SLink$, $p=0.01$ \\
        \input{para_vs_opt_n200_p01_t5_K30_N100_single}
      } & \makecell{
         $\ALink$, $p=0.01$ \\
        \input{para_vs_opt_n200_p01_t5_K30_N100_average}
      } & 
      \makecell{ 
        $\CLink$, $p=0.01$ \\
        \input{para_vs_opt_n200_p01_t5_K30_N100_complete} 
      }
      \\[2em]
      \makecell{
        $\SLink$, $p=0.02$ \\
        \input{para_vs_opt_n200_p02_t5_K30_N100_single}
      } &
      \makecell{
         $\ALink$, $p=0.02$ \\
        \input{para_vs_opt_n200_p02_t5_K30_N100_average} 
      } &
      \makecell{
        $\CLink$, $p=0.02$ \\
        \input{para_vs_opt_n200_p02_t5_K30_N100_complete} 
      } \\[2em]
      \makecell{
        $\SLink$, $p=0.05$ \\
        \input{para_vs_opt_n200_p05_t5_K30_N100_single}
      } &
      \makecell{
         $\ALink$, $p=0.05$ \\
        \input{para_vs_opt_n200_p05_t5_K30_N100_average} 
      } &
      \makecell{
        $\CLink$, $p=0.05$ \\ 
        \input{para_vs_opt_n200_p05_t5_K30_N100_complete} 
      } \\
      \multicolumn{3}{c}{\begin{NoHyper}\ref{arilegend}\end{NoHyper}}
    \end{tabular}
    \caption{Efficacy for $n=200$ and $t=5$ with $p \in \{0.01,0.02,0.05\}$.
      The first axis is the size of the drawn sample.
    }
    \label{fig:efficacy-varying-p}
  \end{center}
\end{figure}

\begin{figure}
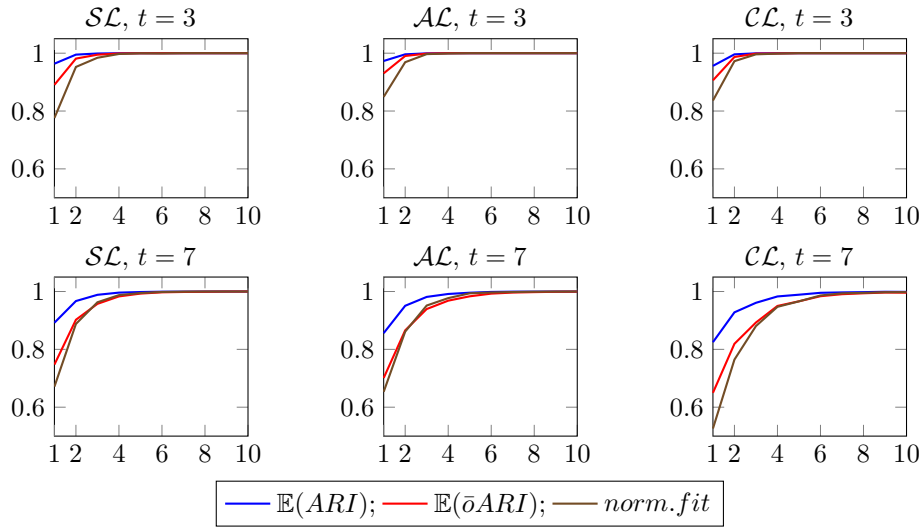

  \begin{center}
    \begin{tabular}{c@{$\qquad$}c@{$\qquad$}c}
      \makecell{
        $\SLink$, $t=3$ \\
        \input{para_vs_opt_n200_p05_t3_K30_N100_single}
      } & \makecell{
         $\ALink$, $t=3$ \\
        \input{para_vs_opt_n200_p05_t3_K30_N100_average}
      } & 
      \makecell{ 
        $\CLink$, $t=3$ \\
        \input{para_vs_opt_n200_p05_t3_K30_N100_complete} 
      }
      \\[2em]
      \makecell{
        $\SLink$, $t=7$ \\
        \input{para_vs_opt_n200_p05_t7_K30_N100_single}
      } &
      \makecell{
         $\ALink$, $t=7$ \\
        \input{para_vs_opt_n200_p05_t7_K30_N100_average} 
      } &
      \makecell{
        $\CLink$, $t=7$ \\ 
        \input{para_vs_opt_n200_p05_t7_K30_N100_complete} 
      } \\
      \multicolumn{3}{c}{\begin{NoHyper}\ref{arilegend}\end{NoHyper}}
    \end{tabular}
    \caption{Efficacy for $n=200$ and $p=0.05$ with $t \in \{3,7\}$.
      The first axis is the size of the drawn sample.
      The plots for $t=5$ can be found in the bottom row of
      Figure~\ref{fig:efficacy-varying-p}.}
    \label{fig:efficacy-varying-t}
  \end{center}
\end{figure}

\begin{figure}
  \begin{center}
    \begin{tabular}{c@{$\qquad$}c@{$\qquad$}c}
      \makecell{
        $\SLink$, $t=10$ \\       
        \input{poly_convergence_n500_p01_t10_K30_N100_single}
      } & \makecell{
        $\ALink$, $t=10$ \\
        \input{poly_convergence_n500_p01_t10_K30_N100_average}
      } & 
      \makecell{ 
        $\CLink$, $t=10$ \\
        \input{poly_convergence_n500_p01_t10_K30_N100_complete}
      }
      \\[2em]
      \makecell{
        $\SLink$, $t=50$ \\
        \input{poly_convergence_n500_p01_t50_K30_N100_single}
      } &
      \makecell{
        $\ALink$, $t=50$ \\
        \input{poly_convergence_n500_p01_t50_K30_N100_average}
      } &
      \makecell{
        $\CLink$, $t=50$ \\
        \input{poly_convergence_n500_p01_t50_K30_N100_complete}
      } \\[2em]
      \makecell{
        $\SLink$, $t=100$ \\
        \input{poly_convergence_n500_p01_t100_K30_N100_single}
      } &
      \makecell{
        $\ALink$, $t=100$ \\
        \input{poly_convergence_n500_p01_t100_K30_N100_average}
      } &
      \makecell{
        $\CLink$, $t=100$ \\ 
        \input{poly_convergence_n500_p01_t100_K30_N100_complete}
      } \\
      \multicolumn{3}{c}{\begin{NoHyper}\ref{convlegend}\end{NoHyper}}
    \end{tabular}
    \caption{Polynomial approximation rate for $n=500$, $P=0.01$ and $t \in \{10,20,40\}$.
      The first axis is the size of the drawn sample.
    }
    \label{fig:approximations-p01}
  \end{center}
\end{figure}

\begin{figure}
  \begin{center}
    \begin{tabular}{c@{$\qquad$}c@{$\qquad$}c}
      \makecell{
        $\SLink$, $t=50$ \\
        \input{poly_convergence_n500_p05_t50_K30_N100_single}
      } &
      \makecell{
        $\ALink$, $t=50$ \\
        \input{poly_convergence_n500_p05_t50_K30_N100_average}
      } &
      \makecell{
        $\CLink$, $t=50$ \\
        \input{poly_convergence_n500_p05_t50_K30_N100_complete}
      } \\[2em]
      \makecell{
        $\SLink$, $t=100$ \\
        \input{poly_convergence_n500_p05_t100_K30_N100_single}
      } &
      \makecell{
        $\ALink$, $t=100$ \\
        \input{poly_convergence_n500_p05_t100_K30_N100_average}
      } &
      \makecell{
        $\CLink$, $t=100$ \\ 
        \input{poly_convergence_n500_p05_t100_K30_N100_complete}
      } \\ 
      \multicolumn{3}{c}{\begin{NoHyper}\ref{convlegend}\end{NoHyper}}
    \end{tabular}
    \caption{Polynomial approximation rate for $n=500$, $p=0.05$ and $t \in \{50,100\}$.
      The first axis is the size of the drawn sample.
    }
    \label{fig:approximations-p05}
  \end{center}
\end{figure}

\begin{figure}
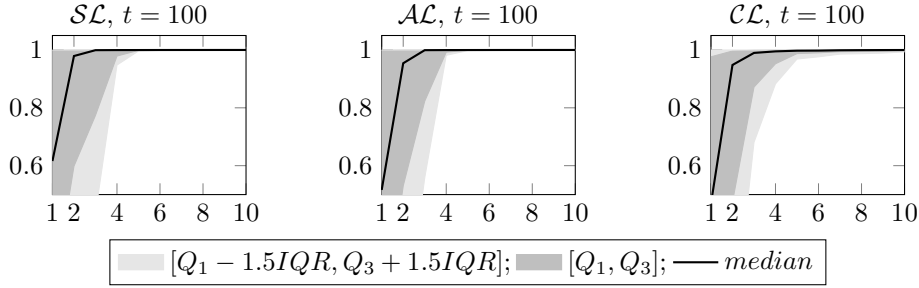

  \begin{center}
    \begin{tabular}{c@{$\qquad$}c@{$\qquad$}c}
      \makecell{
        $\SLink$, $t=100$ \\
        \input{poly_convergence_n500_p10_t100_K30_N100_single}
      } &
      \makecell{
        $\ALink$, $t=100$ \\
        \input{poly_convergence_n500_p10_t100_K30_N100_average}
      } &
      \makecell{
        $\CLink$, $t=100$ \\ 
        \input{poly_convergence_n500_p10_t100_K30_N100_complete}
      } \\ 
      \multicolumn{3}{c}{\begin{NoHyper}\ref{convlegend}\end{NoHyper}}
    \end{tabular}
    \caption{Polynomial approximation rate for $n=500$, $p=0.10$ and $t = 100$.
      The first axis is the size of the drawn sample.
    }
    \label{fig:approximations-p10}
  \end{center}
\end{figure}

\FloatBarrier

\subsubsection{First conclusions}
The first thing that strikes the eye is that the approximations converge very rapidly.
Even for moderately sized spaces ($\sim\!\! 500$ elements), it appears to be
sufficient with $20$ samples for $\SLink$ and $\ALink$, and for smaller spaces
($\sim\!\! 200$ elements), even fewer samples are required.
We also notice that there is a strong correlation between the \mari, \moari{} and normalised fit.

For the part of the demonstration where we have no reference clustering, we cannot know
for sure whether the best reported fit is also optimal. However, from the convergent
behaviour of the data, and the strong correlation between optimality and normalised fit
in Figures~\ref{fig:efficacy-varying-p} and~\ref{fig:efficacy-varying-t}, this points
in the direction of convergence to the true optimum.

Only $\CLink$ displays convergence issues, indicating that if one wishes to use
$\CLink$ for large spaces or large numbers of tied connections, it may be wise to
do so in conjunction with convergence tests.

On the other hand, since $\SLink$ is independent of tie resolution order, 
every sequence of merges ending in the same maximal partition will produce the same partial dendrogram.
This explains why the convergence rate of $\SLink$ is less affected by the
expected number of tied connections than, say, $\CLink$.

The convergence rate is very high in some of the plots of Figures~\ref{fig:approximations-p05}
and~\ref{fig:approximations-p10}. The authors believe this is due the high probability of two
random elements being comparable (high $p$ in $\overline{\RODS}(n,p,t)$), since a dense relation leads
to fewer candidate solutions. This in contrast to the larger set of candidates for a more sparse relation,
such as in Figure~\ref{fig:approximations-p01}.

On the other hand, as we can see in Figures~\ref{fig:approximations-p01}
and~\ref{fig:approximations-p05}, keeping $p$ fixed and increasing the number of tied connections,
and thereby the number of possible branch points, causes a slower convergence rate.

\medskip

To summarise, we see that the approximation is both good and effective for $\SLink$ and $\ALink$.
For $\CLink$, although the approximation method seems good, the required sample size must be increased
in the presence of large amounts of tied connections.

\section{Demonstration on data from the parts database}
\label{section:demo}
While the above demonstration shows that $\HC^{<\Link}_{N,\vareps}$ performs well with respect
to approximating $\HC^{<\Link}_{\opt,\vareps}$,
another question is how order preserving hierarchical clustering deals with the 
dust of reality. In this section, we present results from applying the approximation algorithm
to subsets of the parts database described briefly in Section~\ref{section:motivating-use-case}.
As benchmark, we run classical hierarchical clustering on the same problem instances, comparing
the performance of the methods using \mari, \moari{} and loop frequency (described below).
As hierarchical methods for constrained clustering do not offer a no-link constraint,
we also propose a simplified approach simulating no-link behaviour for $\ALink$
and $\CLink$ which we call $\HC^+$. 

To select data for the demonstration we proceeded as follows: We considered the part-of relations as
a directed graph, and extracted all the connected components. As it turned out, there was one gigantic
component and a large number of singleton elements, but also a hand-full of connected components
of $11$ to $40$ elements each. We selected these smaller connected components as
our demo dataset without any further consideration.
Dissimilarities between the elements were obtained from a dissimilarity measure produced by
an ongoing project in the company working on the very task of classifying equivalent equipment.
Some key characteristics of the data is provided in Table~\ref{tab:cc-characteristics}

\begin{table}[htpb]
  \begin{center}
    \begin{tabular}{c c c c c}
 cc no. & cc size & in/out deg. & $p$ & expected ties \\
\hline
 $0$ & $12$ & $0.92$ & $0.17$ & $2.36$ \\ 
 $1$ & $14$ & $0.93$ & $0.14$ & $4.79$ \\ 
 $2$ & $13$ & $0.92$ & $0.15$ & $2.17$ \\ 
 $3$ & $40$ & $1.27$ & $0.07$ & $8.97$ \\ 
 $4$ & $20$ & $1.35$ & $0.14$ & $3.96$ \\ 
 $5$ & $11$ & $1.18$ & $0.24$ & $2.20$ \\ 
 $6$ & $20$ & $1.10$ & $0.12$ & $4.22$ \\ 
 $7$ & $20$ & $0.95$ & $0.10$ & $3.96$ \\ 
\hline
\end{tabular}
    \caption{Some key characteristics of the connected components selected for the demonstration.
      The in/out deg. column provides the directed average degree when the data is considered as a DAG.
      The column $p$ shows the probability for two random elements to be connected in the transitive
      reduction.}
    \label{tab:cc-characteristics}
  \end{center}
\end{table}

\medskip

Due to limited labeling of the data, we do not know which elements are copies of other elements,
so we have to fake copying to produce planted partitions. For the demonstration, we pick
a connected component $(X^0,E^0)$ where $X^0 = \{x_1^0,\ldots,x_n^0\}$, and for some positive number
$m$ we make $m-1$ copies of $X^0$ and $E^0$, leading to $m$ partially ordered sets
$\big\{(X^k,E^k)\big\}_{k=0}^{m-1}$. We then form their disjoint union $(X,E)$ where
$\card{X} = m \card{X^0}$. $X$ now consists of $m$ connected components, each a copy of the others.
If $x_i^0 \in X^0$, then \deft{the set of elements equivalent to $x_i^0$} is the
set~$\{x_i^k\}_{k=0}^{m-1} \subs X$. Hence, the clusters we seek are the sets on this form.

If we denote the dissimilarity measure that comes with the data by $d_0$, we define the
extension to all of $X$ as follows: First, if both elements are in the same component
$X^k$ for $0 \le k \le m$, then we simply use $d_0$. And if they are in different components,
indicating that they are in a copy-relationship,
we increase their dissimilarity by an offset $\alpha \ge 0$. Concretely,
the extended dissimilarity $d^\alpha : X \times X \to \R_+$ is given by
\begin{equation*} 
  d^\alpha(x_i^r,x_j^s) \ = \
  \begin{cases}
    d_0(x_i^0,x_j^0) & \text{if $r=s$}, \\
    \alpha + d_0(x_i^0,x_j^0) & \text{otherwise}.
  \end{cases}
\end{equation*}
This means that if $x$ and $y$ are copies of each other, then $d^\alpha(x,y)=\alpha$,
and if $x$ and $y$ are in the same component and if $z$ is a copy of $x$,
then $d(z,y) = \alpha + d_0(x,y)$. Furthermore, for each modified distance,
we add a small amount of Gaussian noise to $\alpha$ to induce some variability.
As a result, two copies $x_i^r$ and $x_i^s$ are offset by approximately $\alpha$, and by varying
the magnitude of $\alpha$ we can study how the offset affects the clustering.

\subsection{Simulated constrained clustering}
The available methods for hierarchical constrained clustering do not easily incorporate
the partial order as a constraint. What we would like to compare against, is hierarchical constrained
clustering with do-not-cluster constraints. For $\CLink$ and $\ALink$, we can obtain this
by setting the dissimilarity between comparable elements to a sufficiently large number, causing all
comparable elements to be merged towards the end. Indeed,
for $\CLink$ it is sufficient to set this dissimilarity to any value exceeding $\max\{ d^\alpha \}$,
and as the below demonstration shows, this value works equally well for $\ALink$.
We denote hierarchical clustering with this kind of modified dissimilarity by~$\HC^{+\Link}$.

Since $d_0 < 1$ for all pairs of elements, we chose to use $1.0$ as our maximum dissimilarity.

\subsection{A measure of order preservation}
While the \moari{} measures the correlation between the induced order of the planted partition and the
induced order of the obtained clustering, the \moari{} does not convey information about whether the
induced relation is a partial order or not. Since this is a key question for applications where order
preservation is of high importance (such as acyclic partitioning of graphs),
we suggest the following simple measure.

Let $(Q,E')$ be a partition of $(X,E)$, and let $E_Q$ be the base space projection of $(Q,E')$ onto $X$
(Definition~\ref{def:base-space-projection}). We say that $(Q,E')$ \deft{induces a loop} if there
are elements on the form $(x,x) \in E_Q$. The number of loops induced by $(Q,E')$
is thus the quantity $\card{\{ \, (x,y) \in E_Q \, | \, x = y \,\} }$.
There is at most one loop per element of $X$, and if $E_Q$ contains a cycle, then every element of the
cycle corresponds to a loop. In the name of normalisation, we measure the amount of loops as
the fraction of elements in $X$ that is a part of a cycle:
\[
\loops(Q,E') \ = \ \frac{\card{\{ \, (x,y) \in E_Q \, | \, x = y \,\} }}{\card{X}}.
\]

\subsection{Picking a clustering in the hierarchy for comparison}
Given a problem instance $(X,<,d)$ and a planted partition $Q \in \Part{X,<}$, the planted induced
partial order is necessarily the induced relation $<'$.
But in comparing a hierarchical clustering with a planted partition, we have to make a choice of
clustering in the hierarchy. 
Given a hierarchical clustering, we choose to find the clustering in the hierarchy that has the
highest \mari{} relative the planted partition. We
then report all of \mari, \moari{} and \mloops{} with regards to this clustering.

\subsection{Variance of the difference} \label{section:vardiff}
In the below plots, we present the mean values of \mari, \moari{} and \mloops{} together with a visual
indication of variability. For each instance of a random ordered dissimilarity space $(X,<,d)$,
we run all of $\HC^{<\Link}_{N,\vareps}$, $\HC^\Link$ and $\HC^{+\Link}$.
Thus, we can analyse the performance of the
methods by pairwise comparison on a problem instance level. That is, we choose to consider pairwise
differences such as
\[
\ari(\HC^{<\Link}_{N,\vareps}(X,<,d)) - \ari(\HC^{+\Link}(X,<,d))
\]
as one random variable, and likewise for $\oari$ and $\loops$. The variance of this random variable shows the
variance in the difference, and we can use this magnitude to analyse whether the sets of results are
statistically distinguishable. For the below plots, we mark a region about each line corresponding to one
standard deviation of this random variable. This means that the regions encompassing the lines
will not overlap unless the difference between the mean values is less than two standard deviations.

To reduce the number of plots, we choose to plot the results of all three methods together. This is
obviously impractical with respect to pairwise comparisons, so we employ the following convention:
the indicated variance about the mean of $\HC^{<\Link}_{N,\vareps}$ and $\HC^{+\Link}$ is the standard
deviation of the differences between these methods. The indicated variance about the mean
of $\HC^\Link$ represents the standard deviation of the differences between~$\HC^\Link$
and~$\HC^{<\Link}_{N,\vareps}$.

\subsection{Execution and results}

The parameters given in Table~\ref{table:demo-params} define how the ordered dissimilarity spaces
are constructed for each of the connected components. For each instance of an ordered
dissimilarity space, $\HC^{<\Link}_{N,\vareps}$, $\HC^\Link$ and $\HC^{+\Link}$ are all run on the same
instance with a choice of linkage function $\Link \in \{\SLink,\ALink,\CLink\}$. This allows us to compare
the performance of the methods against each other on a per-instance basis. For each parameter combination
in $\{\alpha\} \times \{\SLink,\ALink,\CLink\}$, we repeated this process $50$ times. The variance of
the difference is based on these sets of $50$ executions.

\begin{table}[tpbh]
  \begin{center}
    \begin{tabular}{lll}
      parameter & value(s) & explanation \\
      \hline
      $\alpha$   & $\{0.10,0.15,\ldots,0.50\}$ & mean copy dissimilarity \\
      $\sigma$   & $0.10$                      & variance of $\alpha$ \\
      $\Link$    & $\{\SLink,\ALink,\CLink\}$  & linkage models \\
      $m$        & $\ceil{200/\card{X^0}}$     & number of copies (see below)\\
      $N$        & $10$                        & sample size in the $N$-fold approximation \\
      $\vareps$  & $10^{-12}$                  & ultrametric completion level \\
      $p$        & $1$                         & choice of norm for ultrametric fitting \\
      \hline
    \end{tabular}
    \caption{Parameters for execution of experiments. The number $m$ of copies is the least number
      for which the total number of elements, $m \card{X^0}$, is at least $200$.}
    \label{table:demo-params}
  \end{center}
\end{table}

We present three families of plots, for \mari, \moari{} and \mloops, respectively
We have picked three connected components for the presentation that we believe represent the span
of observations. The full set of plots is provided in the appendix.

First, connected component number $7$ ($cc7$) is the sample on which we see the most clear benefit from
using $\HC^{<\Link}_{N,\vareps}$, significantly outperforming both $\HC^\Link$ and $\HC^{+\Link}$
on all quality measures.
Although $cc7$ is not representative for the majority of observations, it is empirical evidence
that there exist problem instances for which order preserving clustering cannot be well approximated
by hierarchical constrained clustering through do-not-cluster constraints.

Connected component number $1$ ($cc1$) represents the majority of the instances.
While $\HC^{<\Link}_{N,\vareps}$ still is best in class with respect to all quality measures, we see
that for $\ALink$ and $\CLink$ the method $\HC^{+\Link}$ performs equally well with respect to \mari{} and
sometimes also \moari. 

At the other extreme of $cc7$ there is connected component number $4$ ($cc4$), presented in the bottom row of
Figure~\ref{fig:efficacy-ari}. For this component, all the clustering models perform equally well in
all quality measures, indicating that they produce the exact same clusterings. This can only be
explained by the fact that the original dissimilarity measure $d_0$, when restricted to this
component, both is an ultrametric, and incorporates the order relation (Section~\ref{section:idempotency}).

The results are also summarised in Table~\ref{table:scores} after the plots.

\begin{figure}[htpb]
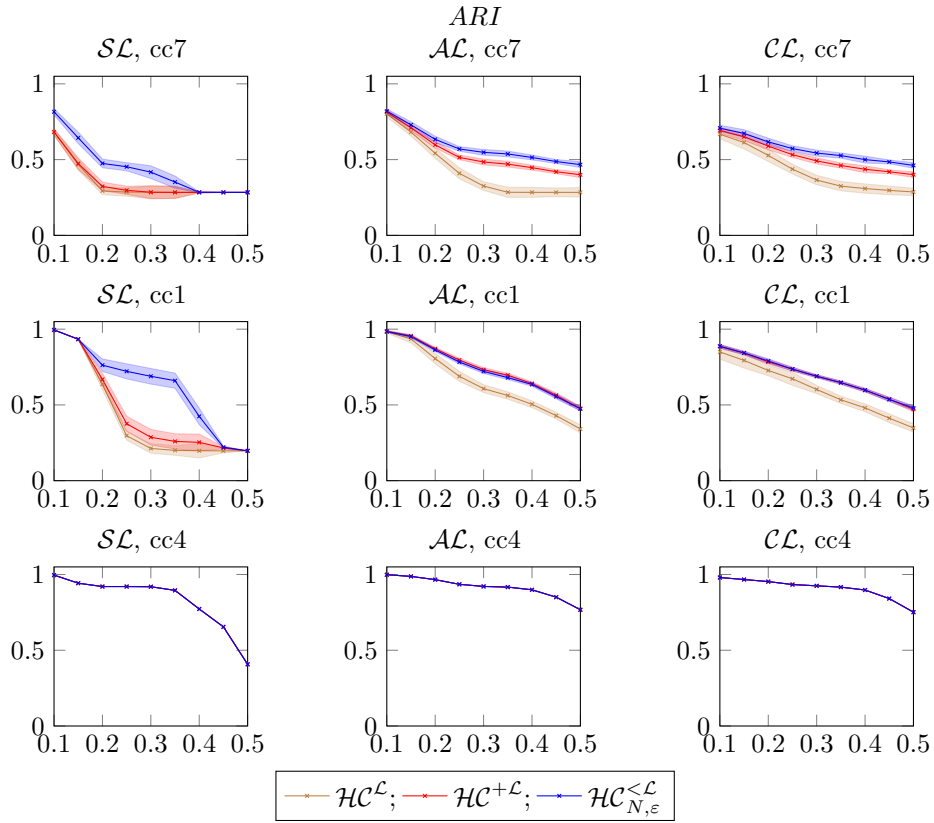

  \begin{center}
    \begin{tabular}{c@{$\qquad$}c@{$\qquad$}c}
      \multicolumn{3}{c}{\mari} \\
      \makecell{$\SLink$, cc$7$ \\  
        \input{cc_7_SL_010_ari}
      } &
      \makecell{$\ALink$, cc$7$ \\
        \input{cc_7_AL_010_ari}
      } &
      \makecell{$\CLink$, cc$7$ \\
        \input{cc_7_CL_010_ari}
      } \\
      \makecell{$\SLink$, cc$1$ \\
        \input{cc_1_SL_010_ari}
      } &
      \makecell{$\ALink$, cc$1$ \\
        \input{cc_1_AL_010_ari}
      } &
      \makecell{$\CLink$, cc$1$ \\
        \input{cc_1_CL_010_ari}
      } \\
      \makecell{$\SLink$, cc$4$ \\
        \input{cc_4_SL_010_ari}
      } &
      \makecell{$\ALink$, cc$4$ \\
        \input{cc_4_AL_010_ari}
      } &
      \makecell{$\CLink$, cc$4$ \\
        \input{cc_4_CL_010_ari}
      } \\ 
      \multicolumn{3}{c}{\begin{NoHyper}\ref{reallegend}\end{NoHyper}}  
    \end{tabular} 
    \caption{Performance of the different clustering methods with respect to \mari{}
      on connected components $7$, $1$ and $4$. The shaded regions represent one standard
      deviation of the pairwise differences, as described in Section~\ref{section:vardiff}.}
    \label{fig:efficacy-ari}
  \end{center}
\end{figure}

\begin{figure}[htpb]
  \begin{center}
    \begin{tabular}{c@{$\qquad$}c@{$\qquad$}c}
      \multicolumn{3}{c}{\moari} \\
      \makecell{$\SLink$, cc$7$ \\  
        \input{cc_7_SL_010_oari}
      } &
      \makecell{$\ALink$, cc$7$ \\
        \input{cc_7_AL_010_oari}
      } &
      \makecell{$\CLink$, cc$7$ \\
        \input{cc_7_CL_010_oari}
      } \\
      \makecell{$\SLink$, cc$1$ \\
        \input{cc_1_SL_010_oari}
      } &
      \makecell{$\ALink$, cc$1$ \\
        \input{cc_1_AL_010_oari}
      } &
      \makecell{$\CLink$, cc$1$ \\
        \input{cc_1_CL_010_oari}
      } \\
      \multicolumn{3}{c}{\begin{NoHyper}\ref{reallegend}\end{NoHyper}}  
    \end{tabular} 
    \caption{Performance of the different clustering methods with respect to \moari{}
      on connected components $7$ and $1$. The shaded regions represent one standard
      deviation of the pairwise differences, as described in Section~\ref{section:vardiff}.}
    \label{fig:efficacy-oari}
  \end{center}
\end{figure}

\begin{figure}[htpb]
  \begin{center}
    \begin{tabular}{c@{$\qquad$}c@{$\qquad$}c}
      \multicolumn{3}{c}{$\loops$} \\
      \makecell{$\SLink$, cc$7$ \\  
        \input{cc_7_SL_010_loops}
      } &
      \makecell{$\ALink$, cc$7$ \\
        \input{cc_7_AL_010_loops}
      } &
      \makecell{$\CLink$, cc$7$ \\
        \input{cc_7_CL_010_loops}
      } \\
      \makecell{$\SLink$, cc$1$ \\
        \input{cc_1_SL_010_loops}
      } &
      \makecell{$\ALink$, cc$1$ \\
        \input{cc_1_AL_010_loops}
      } &
      \makecell{$\CLink$, cc$1$ \\
        \input{cc_1_CL_010_loops}
      } \\
      \multicolumn{3}{c}{\begin{NoHyper}\ref{reallegend}\end{NoHyper}}  
    \end{tabular} 
    \caption{Performance of the different clustering methods with respect to $\loops$
      on connected components $7$ and $1$. The shaded regions represent one standard
      deviation of the pairwise differences, as described in Section~\ref{section:vardiff}.}
    \label{fig:efficacy-loops}
  \end{center}
\end{figure}

\FloatBarrier

We summarise the experiment observations in Table~\ref{table:scores}.
As can be seen from the table, $\HC^{<\Link}_{N,\vareps}$ is best in class in every category.
However, $\HC^{+\Link}$ is also best in class in $81\%$ of the cases when we restrict our attention to
\mari{} and $\Link \in \{\ALink,\CLink\}$.

\begin{table}[tpbh]
  \begin{center}
    \begin{tabular}{|c|c|c|c|c|c|c|c|c|c|}
      \cline{2-10}
      \multicolumn{1}{c|}{\rule{0pt}{2.7ex}}
      & \multicolumn{3}{c|}{$\HC^{<\Link}$}
      & \multicolumn{3}{c|}{$\HC^{\Link}$}
      & \multicolumn{3}{c|}{$\HC^{+\Link}$}
      \\ \cline{2-10}
      \multicolumn{1}{c|}{\rule{0pt}{2.5ex}}
      & \mari & \moari & $\loops$
      & \mari & \moari & $\loops$
      & \mari & \moari & $\loops$
      \\ \hline
            $\SLink$ & $8$ & $8$ & $8$ & $4$ & $4$ & $3$ & $4$ & $4$ & $2$  \\
      $\ALink$ & $8$ & $8$ & $8$ & $2$ & $2$ & $1$ & $7$ & $6$ & $1$  \\
      $\CLink$ & $8$ & $8$ & $8$ & $2$ & $3$ & $0$ & $7$ & $6$ & $1$  \\
      \hline
               & $100 \%$ & $100 \%$ & $100 \%$ & $33 \%$ & $37 \%$ & $16 \%$ & $75 \%$ & $66 \%$ & $16 \%$ 

      \\
      \hline 
    \end{tabular}
    \caption{The table presents for how many of the eight selected samples the different methods
      are best in class with regards to \mari, \moari{} and $\loops$. The scores are based on visual
      inspection of the plots. For \mari{} and \moari, we count a one if there is less than one standard
      deviation to the best plot in at least half the sampled $\alpha$ values and zero otherwise.
      For \mloops, we count a one if the expected value is zero throughout.
      The full list of plots can be found in Appendix~\ref{appendix:plots}.} 
    \label{table:scores}
  \end{center}
\end{table}

To conclude, we see that if clustering is the sole objective, then $\HC^{+\Link}$ is a good alternative
to $\HC^{<\Link}$ whenever $\Link \in \{\ALink,\CLink\}$. If order preservation,
or acyclic partitioning, is of any importance, then $\HC^{<\Link}_{N,\vareps}$ is the only viable method
among those we have tested.

Moreover, as demonstrated by the top row of Figure~\ref{fig:efficacy-ari}, although $\HC^{+\Link}$
may be a good approximation of $\HC^{<\Link}_{N,\vareps}$ when  $\Link \in \{\ALink,\CLink\}$,
there are problem instances on which the latter outperforms the former with significant margin, also for \mari.

\FloatBarrier

\section{Summing up}
\label{section:conclusions}
In this paper we have put forth a theory for order preserving hierarchical agglomerative clustering
for strictly partially ordered sets. The clustering uses classical linkage functions such as single-, average-,
and complete linkage. The clustering is optimisation based, and therefore also permutation invariant.

The output of the clustering process is partial dendrograms;
sub-trees of dendrograms with several connected components.
We have shown that the family of partial dendrograms over a set embed 
into the family of dendrograms over the set.

When applying the theory to non-ordered sets, we see that we have a new theory for
hierarchical agglomerative clustering that is very close to the classical theory,
but that is optimisation based rather than algorithmic.
Differently from classical hierarchical clustering, our theory is permutation
invariant. We have shown that for single linkage, the theory coincides with classical
hierarchical clustering, while for complete linkage, the clustering problem becomes NP-hard.
However, the computational complexity is directly linked to the number of tied connections, and
in the absence of tied connections, the theories coincide.

We present a polynomial approximation scheme for the clustering theory, and demonstrate
its convergence properties and efficacy on randomly generated data. We also provide a demonstration
on real world data comparing against existing methods, showing that our model is best in class in
all selected quality measures.

\subsection{Future work topics}

We suggest the following future work topics:

\subsubsection{Complexity}
\label{section:conclusion-complexity}
While NP-hardness of $\HC^{<\CLink}_{\opt,\vareps}$ follows from Theorem~\ref{thm:CL-is-NP-hard},
the complexity classes of order preserving hierarchical agglomerative clustering for
$\SLink$ and $\ALink$ remain to be established.

\subsubsection{Order versus dissimilarity}

Since the order relation is treated as a binary constraint it has a significant effect on the output from the
clustering process, and may in some cases lead to undesirable outcomes.
For example, if the dissimilarity measure associates ``wrong'' elements for clustering,
the induced order relation may exclude future merges of elements correctly belonging
together by erroneously identifying them as comparable. Also, if elements are wrongly
identified as comparable to begin with, they can never be merged. Both due to
Theorem~\ref{thm:one-point-quotient}.

Together, these observations indicate that ``loosening up'' the stringent
nature of the order relation may be beneficial in applications where order preservation is not
a strict requirement.

\paragraph{Acknowledgments.}
I would also like to thank the anonymous reviewers at Machine Learning for
constructive feedback and comments greatly improving the exposition.
I would also like to thank Henrik Forssell, Department of Informatics (University of Oslo),
and Gudmund Hermansen, Department of Mathematics (University of Oslo), for their comments,
questions and discussions leading up to this work.

\appendix

\section{Plots from the part database demo}
\label{appendix:plots}

This section lists all the plots from the experiments described in Section~\ref{section:demo}.
The plots are grouped per connected component,
and present results for all clustering methods, quality measures and linkage models.
Please see Table~\ref{tab:cc-characteristics} for a list of statistical properties of
the different connected components, and Table~\ref{table:demo-params} for the parameter settings used
during the experiments.

\begin{longtable}[c]{c@{$\qquad$}c@{$\qquad$}c}
  \multicolumn{3}{p{\textwidth}}{%
    Results for connected component no.~$0$.} \\
  $\quad\SLink$ & $\ALink$ & $\CLink$ \\
  \endfirsthead%
  \multicolumn{3}{p{\textwidth}}{%
    Continued: results for connected component no.~$0$.} \\
  $\quad\SLink$ & $\ALink$ & $\CLink$ \\
  \endhead%
  \multicolumn{3}{c}{\begin{NoHyper}\ref{reallegend}\end{NoHyper}}  \\
  \multicolumn{3}{r}{continued on next page$\ldots$}%
  \endfoot%
  \multicolumn{3}{c}{\begin{NoHyper}\ref{reallegend}\end{NoHyper}}%
  \endlastfoot%
  \input{cc_0_SL_010_ari_ylabel} &
  \input{cc_0_AL_010_ari} &
  \input{cc_0_CL_010_ari} \\
  \input{cc_0_SL_010_oari_ylabel} &
  \input{cc_0_AL_010_oari} &
  \input{cc_0_CL_010_oari} \\ 
  \input{cc_0_SL_010_loops_ylabel} &
  \input{cc_0_AL_010_loops} &
  \input{cc_0_CL_010_loops} 
\end{longtable}

\begin{longtable}[c]{c@{$\qquad$}c@{$\qquad$}c}
  \multicolumn{3}{p{\textwidth}}{%
    Results for connected component no.~$1$.} \\
  $\quad\SLink$ & $\ALink$ & $\CLink$ \\
  \endfirsthead%
  \multicolumn{3}{p{\textwidth}}{%
    Continued: results for connected component no.~$1$.} \\
  $\quad\SLink$ & $\ALink$ & $\CLink$ \\
  \endhead%
  \multicolumn{3}{c}{\begin{NoHyper}\ref{reallegend}\end{NoHyper}}  \\
  \multicolumn{3}{r}{continued on next page$\ldots$}%
  \endfoot%
  \multicolumn{3}{c}{\begin{NoHyper}\ref{reallegend}\end{NoHyper}}%
  \endlastfoot%
  \input{cc_1_SL_010_ari_ylabel} &
  \input{cc_1_AL_010_ari} &
  \input{cc_1_CL_010_ari} \\
  \input{cc_1_SL_010_oari_ylabel} &
  \input{cc_1_AL_010_oari} &
  \input{cc_1_CL_010_oari} \\ 
  \input{cc_1_SL_010_loops_ylabel} &
  \input{cc_1_AL_010_loops} &
  \input{cc_1_CL_010_loops} 
\end{longtable}

\begin{longtable}[c]{c@{$\qquad$}c@{$\qquad$}c}
  \multicolumn{3}{p{\textwidth}}{%
    Results for connected component no.~$2$.} \\
  $\quad\SLink$ & $\ALink$ & $\CLink$ \\
  \endfirsthead%
  \multicolumn{3}{p{\textwidth}}{%
    Continued: results for connected component no.~$2$.} \\
  $\quad\SLink$ & $\ALink$ & $\CLink$ \\
  \endhead%
  \multicolumn{3}{c}{\begin{NoHyper}\ref{reallegend}\end{NoHyper}}  \\
  \multicolumn{3}{r}{continued on next page$\ldots$}%
  \endfoot%
  \multicolumn{3}{c}{\begin{NoHyper}\ref{reallegend}\end{NoHyper}}%
  \endlastfoot%
  \input{cc_2_SL_010_ari_ylabel} &
  \input{cc_2_AL_010_ari} &
  \input{cc_2_CL_010_ari} \\
  \input{cc_2_SL_010_oari_ylabel} &
  \input{cc_2_AL_010_oari} &
  \input{cc_2_CL_010_oari} \\ 
  \input{cc_2_SL_010_loops_ylabel} &
  \input{cc_2_AL_010_loops} &
  \input{cc_2_CL_010_loops} 
\end{longtable}

\begin{longtable}[c]{c@{$\qquad$}c@{$\qquad$}c}
  \multicolumn{3}{p{\textwidth}}{%
    Results for connected component no.~$3$.} \\
  $\quad\SLink$ & $\ALink$ & $\CLink$ \\
  \endfirsthead%
  \multicolumn{3}{p{\textwidth}}{%
    Continued: results for connected component no.~$3$.} \\
  $\quad\SLink$ & $\ALink$ & $\CLink$ \\
  \endhead%
  \multicolumn{3}{c}{\begin{NoHyper}\ref{reallegend}\end{NoHyper}}  \\
  \multicolumn{3}{r}{continued on next page$\ldots$}%
  \endfoot%
  \multicolumn{3}{c}{\begin{NoHyper}\ref{reallegend}\end{NoHyper}}%
  \endlastfoot%
  \input{cc_3_SL_010_ari_ylabel} &
  \input{cc_3_AL_010_ari} &
  \input{cc_3_CL_010_ari} \\
  \input{cc_3_SL_010_oari_ylabel} &
  \input{cc_3_AL_010_oari} &
  \input{cc_3_CL_010_oari} \\ 
  \input{cc_3_SL_010_loops_ylabel} &
  \input{cc_3_AL_010_loops} &
  \input{cc_3_CL_010_loops} 
\end{longtable}

\begin{longtable}[c]{c@{$\qquad$}c@{$\qquad$}c}
  \multicolumn{3}{p{\textwidth}}{%
    Results for connected component no.~$4$.} \\
  $\quad\SLink$ & $\ALink$ & $\CLink$ \\
  \endfirsthead%
  \multicolumn{3}{p{\textwidth}}{%
    Continued: results for connected component no.~$4$.} \\
  $\quad\SLink$ & $\ALink$ & $\CLink$ \\
  \endhead%
  \multicolumn{3}{c}{\begin{NoHyper}\ref{reallegend}\end{NoHyper}}  \\
  \multicolumn{3}{r}{continued on next page$\ldots$}%
  \endfoot%
  \multicolumn{3}{c}{\begin{NoHyper}\ref{reallegend}\end{NoHyper}}%
  \endlastfoot%
  \input{cc_4_SL_010_ari_ylabel} &
  \input{cc_4_AL_010_ari} &
  \input{cc_4_CL_010_ari} \\
  \input{cc_4_SL_010_oari_ylabel} &
  \input{cc_4_AL_010_oari} &
  \input{cc_4_CL_010_oari} \\ 
  \input{cc_4_SL_010_loops_ylabel} &
  \input{cc_4_AL_010_loops} &
  \input{cc_4_CL_010_loops} 
\end{longtable}

\begin{longtable}[c]{c@{$\qquad$}c@{$\qquad$}c}
  \multicolumn{3}{p{\textwidth}}{%
    Results for connected component no.~$5$.} \\
  $\quad\SLink$ & $\ALink$ & $\CLink$ \\
  \endfirsthead%
  \multicolumn{3}{p{\textwidth}}{%
    Continued: results for connected component no.~$5$.} \\
  $\quad\SLink$ & $\ALink$ & $\CLink$ \\
  \endhead%
  \multicolumn{3}{c}{\begin{NoHyper}\ref{reallegend}\end{NoHyper}}  \\
  \multicolumn{3}{r}{continued on next page$\ldots$}%
  \endfoot%
  \multicolumn{3}{c}{\begin{NoHyper}\ref{reallegend}\end{NoHyper}}%
  \endlastfoot%
  \input{cc_5_SL_010_ari_ylabel} &
  \input{cc_5_AL_010_ari} &
  \input{cc_5_CL_010_ari} \\
  \input{cc_5_SL_010_oari_ylabel} &
  \input{cc_5_AL_010_oari} &
  \input{cc_5_CL_010_oari} \\ 
  \input{cc_5_SL_010_loops_ylabel} &
  \input{cc_5_AL_010_loops} &
  \input{cc_5_CL_010_loops} 
\end{longtable}

\begin{longtable}[c]{c@{$\qquad$}c@{$\qquad$}c}
  \multicolumn{3}{p{\textwidth}}{%
    Results for connected component no.~$6$.} \\
  $\quad\SLink$ & $\ALink$ & $\CLink$ \\
  \endfirsthead%
  \multicolumn{3}{p{\textwidth}}{%
    Continued: results for connected component no.~$6$.} \\
  $\quad\SLink$ & $\ALink$ & $\CLink$ \\
  \endhead%
  \multicolumn{3}{c}{\begin{NoHyper}\ref{reallegend}\end{NoHyper}}  \\
  \multicolumn{3}{r}{continued on next page$\ldots$}%
  \endfoot%
  \multicolumn{3}{c}{\begin{NoHyper}\ref{reallegend}\end{NoHyper}}%
  \endlastfoot%
  \input{cc_6_SL_010_ari_ylabel} &
  \input{cc_6_AL_010_ari} &
  \input{cc_6_CL_010_ari} \\
  \input{cc_6_SL_010_oari_ylabel} &
  \input{cc_6_AL_010_oari} &
  \input{cc_6_CL_010_oari} \\ 
  \input{cc_6_SL_010_loops_ylabel} &
  \input{cc_6_AL_010_loops} &
  \input{cc_6_CL_010_loops} 
\end{longtable}

\begin{longtable}[c]{c@{$\qquad$}c@{$\qquad$}c}
  \multicolumn{3}{p{\textwidth}}{%
    Results for connected component no.~$7$.} \\
  $\quad\SLink$ & $\ALink$ & $\CLink$ \\
  \endfirsthead%
  \multicolumn{3}{p{\textwidth}}{%
    Continued: results for connected component no.~$7$.} \\
  $\quad\SLink$ & $\ALink$ & $\CLink$ \\
  \endhead%
  \multicolumn{3}{c}{\begin{NoHyper}\ref{reallegend}\end{NoHyper}}  \\
  \multicolumn{3}{r}{continued on next page$\ldots$}%
  \endfoot%
  \multicolumn{3}{c}{\begin{NoHyper}\ref{reallegend}\end{NoHyper}}%
  \endlastfoot%
  \input{cc_7_SL_010_ari_ylabel} &
  \input{cc_7_AL_010_ari} &
  \input{cc_7_CL_010_ari} \\
  \input{cc_7_SL_010_oari_ylabel} &
  \input{cc_7_AL_010_oari} &
  \input{cc_7_CL_010_oari} \\ 
  \input{cc_7_SL_010_loops_ylabel} &
  \input{cc_7_AL_010_loops} &
  \input{cc_7_CL_010_loops} 
\end{longtable}

\FloatBarrier

\section{Reference implementation}
The implementation used for the experiments in Sections~\ref{section:approximation-demo}
and~\ref{section:demo} is available as open source at \url{https://bitbucket.org/Bakkelund/ophac}.


\bibliography{preamble}

\end{document}